%% file: main.tex
\documentclass{article}
\usepackage{stat260_format}
\usepackage{stat260_commands}
\usepackage{physics}

\usepackage{libertine} 
\usepackage[numbers]{natbib}

\usepackage{microtype}
\usepackage{graphicx}
\usepackage{subcaption}
\usepackage{booktabs} 
\usepackage{xspace}

\usepackage{xcolor}
\usepackage{wrapfig}
\usepackage{sidecap}
\usepackage[normalem]{ulem}
\usepackage[margin=1in]{geometry}

\usepackage{hyperref}
\hypersetup{unicode=true}
\definecolor{mydarkblue}{rgb}{0,0.08,0.45}

\usepackage{enumitem}
\usepackage{amsmath}
\usepackage{amssymb}
\usepackage{mathtools}
\usepackage{amsthm}
\usepackage[normalem]{ulem}

\usepackage[capitalize,noabbrev]{cleveref}

\newcommand{\data}{{\mathcal{D}}_n}
\newcommand{\Xtrain}{\bX}
\newcommand{\ytrain}{\by}
\newcommand{\testloss}{R}
\newcommand{\trainloss}{\widehat{R}_n}
\newcommand{\ofitlevel}{\tau}
\newcommand{\noise}{\bxi_{\mathrm{noise}}}
\newcommand{\effnoise}{\bxi}
\newcommand{\lin}{\mathrm{lin}}
\newcommand{\linnoise}{\bxi_{\lin}}
\newcommand{\excess}{\cE}
\newcommand{\minexcess}{\cE_\star}
\newcommand{\auxmin}{\cE^{\lin}_{\star}}
\newcommand{\st}{\text{ s.t. }}
\newcommand{\lmult}{\eta}
\newcommand{\ridgeparam}{\lambda}
\newcommand{\MP}{\mathsf{MP}}
\newcommand{\midsmall}{\,|\,}
\newcommand{\diag}{\text{diag}}

\newcommand{\revised}[1]{\textcolor{black}{#1}}
\newcommand{\rev}[1]{\textcolor{black}{#1}}
\newcommand{\remove}[1]{}

\title{A Universal Trade-off Between the Model Size,  Test Loss, and Training Loss of Linear Predictors} 
\author{Nikhil Ghosh\\
Department of Statistics\\University of California, Berkeley\\
\texttt{\sf nikhil\_ghosh@berkeley.edu}
\and
\and
Mikhail Belkin\\
Hal{\i}c{\i}o{\u g}lu Data Science Institute\\
University of California, San Diego\\
\texttt{\sf mbelkin@ucsd.edu}
}
\date{}

\begin{document}

\maketitle

\begin{abstract}
\input{sections/abstract}
\end{abstract}

\input{sections/intro}
\input{sections/related}
\input{sections/general_revised}
\input{sections/asymptotic}
\input{sections/conclusion}
\input{sections/acknowledgments}

\newpage
\bibliographystyle{unsrt}
\bibliography{refs}

\newpage
\appendix
\onecolumn
\input{sections/appendix}

\end{document}

%% file: sections/abstract.tex
In this work we establish an algorithm and distribution independent non-asymptotic trade-off between the model size, excess test loss, and training loss of linear predictors. Specifically, we show that models that perform well on the test data (have low excess loss) are either  ``classical'' -- have training loss close to the noise level,  or are ``modern'' -- have a much larger number of parameters compared to the minimum needed to fit the training data exactly. 

We also provide a more precise asymptotic analysis when the limiting spectral distribution of the whitened features is Marchenko-Pastur. Remarkably, while the Marchenko-Pastur analysis is far more precise near the interpolation peak, where the number of parameters is just enough to fit the training data, it coincides exactly with the distribution independent bound as the level of overparametrization increases.




%% file: sections/intro.tex
\section{Introduction}
Classical statistics and machine learning models have  traditionally been analyzed in regimes where the training loss (also known as the empirical risk) approximates the test loss. In contrast, many modern deep learning systems 
obtain a much lower loss on the training set than on the test set. In recent years there has been a growing theoretical understanding that different models that fit noisy data perfectly (interpolate the data) can nevertheless generalize optimally or nearly optimally. This phenomenon, which has come to be known as ``benign overfitting''~\cite{bartlett2020benign} or ``harmless interpolation''~\cite{muthukumar2020harmless} can be shown to provably occur in a wide array of settings including for non-parametric weighted nearest neighbor type methods~\cite{devroye1998hilbert, belkin2018overfitting,belkin2019does,chhor2022benign}, linear regression~\cite{hastie2022surprises, bartlett2020benign, muthukumar2020harmless, belkin2020two, koehler2021uniform, tsigler2020benign, chatterji2022foolish, zou2021benign}, random features and kernel methods~\cite{mei2022random, mei2022generalization, liang2020just, adlam2020neural}, and neural networks~\cite{cao2022benign, frei2022benign, frei2023benign, cao2022benign, li2021towards}, to give just some representative examples of recent literature. Each of these works reveal a set of sufficient conditions on the data and learning algorithm for which benign overfitting is possible.  However, a common aspect is  overparametrization: the number of model parameters is significantly larger than what is required to fit all the training data.

Indeed, a striking feature of many current models is their size, reaching billions or even trillions parameters~\cite{bommasani2021opportunities}. A clue to understanding the need for a large number of parameters is provided by the double descent generalization curve proposed in~\cite{belkin2019reconciling} which qualitatively describes the relation between model size and its test performance. The shape of the curve suggests that interpolating models need to be significantly over-parameterized compared to the ``interpolation threshold'' (the minimum number of parameters needed to fit the training data)  to achieve near-optimal performance. However, despite the demonstration of double descent and benign overfitting for many specific settings, there is little literature which seeks to understand {\it necessary} conditions for such phenomena to occur. The most significant step in this direction was made by Holzm{\"u}ller~\cite{holzmuller2020universality} which shows that double descent is universal for minimum norm (ridgeless) linear regression and implies the necessity of overparameterization for these interpolating models to have good performance.


Yet, in practice models are 
trained using iterative (gradient-based) methods which are typically stopped early, well before convergence to a truly interpolating solution. Indeed, pushing models to fit the data perfectly can be prohibitively computationally expensive and is usually unnecessary.
Thus, while the interpolation analyses are insightful for understanding  modern machine learning, they are a limit case for the settings of most practical interest. 

In this paper, we aim to shed light on these issues by analyzing the connection between the {\it training (empirical) loss}, the {\it expected (test) loss} and the {\it number of parameters} for general linear models.  Specifically, we show that for any algorithm producing a linear predictor with near-optimal expected loss, the output model  
must either be ``classical'' with the training loss relatively large -- close to the noise level, or have a large excess over-parameterization (i.e., the number of parameters relative to the minimum necessary to just fit the training data without concern for generalization)\footnote{ The``classical'' vs ``modern'' distinction is based on the empirical loss rather than the number of parameters. Thus, for models rich enough to fit the training data perfectly, it is a consequence of the training algorithm rather than an inherent property of the model as such. Indeed, models in classical settings can still be highly parametric, as is the case for the traditional analyses of kernel machines (e.g.,~\cite{steinwart2008support}), which can be viewed as infinite-dimensional linear models.}. The trade-off between the empirical loss and the number of parameters is universal in the sense that it holds for \textit{any algorithm} which outputs a linear predictor using \textit{any} non-degenerate feature map for \textit{any} regression problem with noise. 

Furthermore, we provide a more precise analysis under 
additional distributional and asymptotic assumptions. Remarkably, the universal bound is tight (up to constant factors) for this far more special ``asymptotic Gaussian'' case when the model is sufficiently over-parameterized and obtains training loss strictly below the noise level.

To introduce the setting of interest, consider a general regression problem where we are given a training set of $n$ samples $\data = \{(\bx_i, y_i)\}_{i=1}^n$ with $\bx_i \in \R^d$ and $y_i \in \R$. We assume that each point is sampled $(\bx_i, y_i) \simiid P$ for some distribution $P$ on $\R^d \times \R$. 
For a function $f: \R^d \to \R$ we define its training (empirical) loss $\trainloss(f)$ and its expected (test) loss $\testloss(f)$ as
\[
    \trainloss(f) := \frac{1}{n} \sum\limits_{i=1}^n (y - f(\bx_i))^2,~~~~ \testloss(f) := \E_{(\bx, y)}(y - f(\bx))^2.
\]
The {\it regression function} is defined as 
$
f_\star(\bx) := \E(y \midsmall \bx)
$.
It is well-known that $f_\star(\bx)$
is the optimal predictor for regression in the sense of minimizing the expected loss:
\[
f_\star(\bx) = \arg\min_{f} \testloss(f).
\]
Thus for an arbitrary predictor $f$, it makes sense to consider the {\it excess loss}  
\[
\cE(f) := \testloss(f) - \testloss(f_\star)
\]
as a measure of the performance of $f$ compared to the best theoretically achievable test loss. Finally we assume that the problem has  noise level of at least $\sigma^2$, that is for almost all $\bx$
$$
\Var(y \midsmall \bx) \geq \sigma^2.
$$
Note that the last condition directly implies that  $\testloss(f_\star) \ge \sigma^2$. We now consider a general $p$-dimensional linear  feature model
$\phi_p: \R^d \to \R^p$  of the form
\[
    \beta(\bx) = \bbeta^\sT \phi_p(\bx), ~~\bbeta\in\R^p.
\]
Here the map $\phi_p$ can be deterministic or random. 
The optimal  linear predictor $\beta_\star$ is given by 
$$
\beta_\star = \arg\min_{\bbeta \in \R^p} R(\beta).
$$
It is clear  that the excess loss of $\beta$ compared to best predictor $f_\star$
is bounded from below by the excess loss of $\beta$ compared to the best {\it linear} predictor $\beta_\star$
$$
\excess(\beta) \ge \excess^{\lin} (\beta) := \testloss(\beta) - \testloss(\beta_\star). 
$$
Assume now that we have an algorithm $\mathcal A$ that
 given the training data $\data$,
 outputs a linear predictor $\beta=\cA(\data)$,  with the empirical loss  bounded by $\ofitlevel \ge 0$ relative to the noise level almost surely i.e., $$\trainloss(\cA(\data))/\sigma^2 \le \ofitlevel.$$
Additionally, assume that $\phi_p$ is non-degenerate (see \cref{sec:main} for the exact conditions). 
Our main result  is the following lower bound on the expected excess loss
\begin{equation}\label{eq:main}
\E_{\data\sim P^n} ~\excess(\cA(\data)) \geq \E_{\data\sim P^n} ~\excess^{\lin}(\cA(\data)) \geq \begin{cases} 
\sigma^2 \frac{n}{p} \revised{(1-\sqrt{\tau})^2}, ~~&\ofitlevel < 1\\
0 &\ofitlevel\ge 1\\
\end{cases}
\end{equation}


{\remark The case of $\ofitlevel \ge 1$ necessarily results in a trivial bound. Indeed, suppose that $\Var(y \midsmall \bx) = \sigma^2$, and the optimal predictor is linear. Then the \revised{``oracle"} algorithm that \revised{always} outputs \revised{$\beta_\star$} for any input has expected training loss  $\E_{\data\sim P^n}~ [\trainloss(\beta_\star)] = \sigma^2$ while its \revised{expected} excess loss $\E_{\data\sim P^n}~[\excess(\beta_\star)]=0$. \revised{In general, our lower bound applies to any algorithm (even ones with knowledge of $\beta_\star$) and for any problem instance. In contrast, minimax results lower bound the performance of any algorithm on worst case problem instances.}}

While the bound above is very general, a more precise analysis is possible asymptotically  under additional distributional  assumptions. Assume that $n,p\to \infty$ with limiting ratio $\lim{n/p}=\gamma$ and that the limiting spectral distribution is Marchenko-Pastur with aspect ratio $\gamma$ (see \cref{sec:app-mp}), which is the case if for example the covariates are Gaussian. Additionally assume that the true model is linear with additive noise obeying certain mild moment conditions (see \cref{sec:mp} for details). Define  
\[
\minexcess(\ofitlevel) = \inf_{\cA \st \trainloss(\cA(\data)) \leq \ofitlevel \sigma^2} \E_{\data\sim P^n}~ [\excess(\cA(\data))]
\]
i.e., $\minexcess(\ofitlevel)$ is the minimal expected excess loss for any algorithm with training loss at most $\ofitlevel \sigma^2$. \revised{In this setting, it turns out that the bound in Eq.\ (\ref{eq:main}) becomes tight when $p \gg n$ namely,}
\begin{equation}\label{eq:mp-mainMP}
\revised{\minexcess(\ofitlevel) \sim} \; 
\sigma^2 \frac{n}{p} (1 - \ofitlevel^{1/2})^2 \revised{\text{ as } p / n \to \infty,} ~~\ofitlevel \in[0,1].
\end{equation}
For small $\ofitlevel$ and $\gamma < 1$ the bound can be improved to the following \footnote{Note that in the limit $\frac{n}{p-n} \sim \frac{\gamma}{1-\gamma}$ and $\frac{p}{p-n} \sim \frac{1}{1-\gamma}$.}:  
\begin{equation}\label{eq:mp-mainMP-smalltau}
 \minexcess(\ofitlevel) \ge \sigma^2\frac{n}{p-n} \qty(1 - {\ofitlevel}^{1/2}\sqrt{\frac{p}{p-n}})^2, ~~\ofitlevel \in [0, 1-\gamma].
\end{equation}
In fact, the bound is tight up to an $o(\sqrt{\ofitlevel})$ error term. Furthermore, under the same conditions, at the interpolation peak $p = n$, we have the following precise expression for the minimal expected excess loss
\begin{equation} \label{eq:mp-mainMP-gamma-one}
{\minexcess(\ofitlevel)} ={\sigma^2}\qty( \frac{1}{4 \ofitlevel} + \frac{\ofitlevel}{4} - \frac{1}{2}).
\end{equation}
A few observations are now in order.
\paragraph{Comparison at interpolation ($\ofitlevel=0$).} A special case of our general result \cref{eq:main} is for $\ofitlevel = 0$, namely when we only consider  models that  interpolate the data. It  is instructive to compare our bound with some of the existing work in that setting.  For $\ofitlevel=0$   we obtain a lower bound for the excess loss of $\sigma^2 \cdot n /p$ which matches the result in~\cite{muthukumar2020harmless} (Corollary 1). The result in~\cite{muthukumar2020harmless} is given for the well-specified linear setting and requires specific covariate assumptions such as Gaussianity, but holds with high probability rather than just in expectation.

A  lower bound for minimum norm interpolating  linear models without distributional assumptions  is  given in~\cite{holzmuller2020universality}. The  bound is of the form $\sigma^2 \cdot n/(p-n+1)$ and is significantly tighter near the interpolation peak $p=n$. Remarkably, their general bound which holds under minimal assumptions, almost matches the exact computation for the Gaussian case in~\cite{belkin2020two} which yields $\sigma^2 \cdot n /(p-n-1)$ for $p\ge n+2$.  
{\remark We note that while the results in~\cite{holzmuller2020universality} are stated for minimum norm predictors, their analysis implies an algorithm independent lower bound for interpolating models, which  is sharper than our bound in \cref{eq:main} for $\ofitlevel=0$.}


\paragraph{Comparison between interpolating and non-interpolating regimes.}
We will now compare the interpolating regime ($\ofitlevel=0$) with the non-interpolating regime ($\ofitlevel>0$). 

\noindent{\bf a. The peak behavior ($p \approx n$).} The general results in~\cite{holzmuller2020universality} demonstrate a sharp peak at the interpolation threshold $p=n$. Indeed, the analysis  for the Gaussian setting \cite{belkin2020two} shows that the peak is in fact infinite. Note however, any non-zero regularization attenuates the peak, making it finite (e.g.,~\cite{mei2022generalization}). Note that $\ofitlevel$ can also be viewed as regularization. In the asymptotic MP setting \cref{eq:mp-mainMP-gamma-one} shows that the expected loss has a pole singularity $\ofitlevel^{-1}$ at the peak.  Thus the transition between interpolating and non-interpolating regimes is discontinuous in terms of the height of the interpolation peak. Hence we see that our general bound in \cref{eq:main} is loose close to the peak, which is to be expected as it is continuous in $\ofitlevel$ at $0$, while the actual expected loss is discontinuous. In contrast, the Marchenko-Pastur setting bound in \cref{eq:mp-mainMP-smalltau} is much more accurate. 

\noindent{\bf b. The ``tail'' behavior ($p\gg n$)}. 
In sharp contrast to the peak, our general bound show that the ``tail'' ($p\gg n$) behavior of the generalization curve is remarkably stable with respect to $\ofitlevel$. 
Achieving nearly optimal excess risk requires either $\ofitlevel > 1-o(1)$ or $p\gg n$ (or both, of course). Thus interpolating and non-interpolating solutions require essentially the same level of over-parameterization to approach optimality as long as the loss of non-interpolating models is at least slightly lower than the noise level.  

Furthermore, for any empirical loss smaller than the noise level ($\ofitlevel<1$),  when $p\gg n$ the general bound in \cref{eq:main}  matches the asymptotic analysis in the Marchenko-Pastur setting \remove{up to a multiplicative constant}. This is remarkable, as the general bound is not asymptotic and makes essentially no assumptions on the covariate distribution, regression function, or the structure of the noise, \revised{yet it is still tight in the limit of increasing overparametrization}.

While the interpolation peak is a striking feature of the generalization curve, the tail behavior is arguably  more important for understanding practical applications. Indeed, the tail behaviour seems consistent with over-parameterization in practical models which are routinely trained so that the training loss is significantly lower than the test loss but far from zero. In contrast, the peak is a less robust phenomenon which describes only specific regime of training and is highly sensitive to the presence of regularization.
\paragraph{Convergence rates.}
A classical line of statistical analysis is concerned with convergence rates for various estimation problems. 
Typically \revised{statistical rates for} regression (see e.g.,~\cite{tsybakov2009intro}) are of the form 
$$
\excess (\hat{f}_n)  = O(n^{-\alpha}),~~~ \alpha >0 
$$
where $\hat{f}_n$ is a predictor based on a training set with $n$ samples \rev{and $\alpha$ may depend on some notion of data dimensionality, such as the dimension of the data manifold}.

We note that such parametric or non-parametric  rates are easily compatible with our analysis. 
As a corollary of our lower bound \cref{eq:main}, we can see that in order to achieve such a rate with a linear model
the model either needs to have excess over-parameterization inversely proportional to the rate, i.e., $p = \Omega(n^{1+\alpha})$ or to be in the ``classical regime" where $\ofitlevel = 1 - o_n(1)$ i.e., the train error essentially at the noise level. \rev{Note that in general (e.g., for a random feature model ~\cite{rahimi2007random}) the number of features $p$ is a property of the model and is distinct from any notion of data dimensionality.} Moreover, note that to simply achieve low training error it is only necessary that $p = \Omega(n)$ \rev{as only $p = n$ features are necessary to interpolate the training data}, hence the condition $p = \Omega(n^{1+\alpha})$ is additionally requiring at least $n^\alpha$ times the number of parameters to achieve the desired excess loss rate. 

Interestingly, there are also settings where near-interpolation is necessary to approach optimal generalization~\cite{cheng2022memorize, feldman2020does, brown2021memorization}. Our results imply that in these settings significant excess over-parameterization is unavoidable. \revised{Of the aforementioned works, the most closely related to ours work \cite{cheng2022memorize} which studies high-dimensional linear regression. Moreoever, they study the optimal test loss subject to a training loss constraint as we do in this paper. However, a key difference is that their results hold in a Bayesian setting where the true model is drawn from some prior distribution and all losses are averaged over this prior whereas our results hold even for a fixed target function. In particular, only in our setting can an estimator achieve zero excess loss using finitely many samples.}


Finally, we note that the trade-off presented in this paper is reminiscent of the trade-off between smoothness and over-parameterization discovered in~\cite{bubeck2021universal}, which shows that over-parameterization is necessary to fit noisy data smoothly. 
In contrast to~\cite{bubeck2021universal} which does not consider generalization, predictors that generalizes well while fitting noise need to be ``spiky'' rather than smooth, hence over-parameterization in our paper serves a different function. 

%% file: sections/related.tex

%% file: sections/general_revised.tex
\section{Universal Lower Bound}\label{sec:main} 
First let us introduce some notation. We define the training data matrices 
\[
\Xtrain := \left(\bx_1^\sT, \ldots, \bx_n^\sT\right) \in \R^{n \times d},~~~~ \ytrain := (y_1, \ldots, y_n) \in \R^n
\]
so that $\data = (\Xtrain, \ytrain) \sim P^n$. If $(\bx, y) \sim P$ then we will use $P_X$ to denote the marginal distribution on $\bx$. We will use the following abbreviated notation for the expectations
\[
    \E_{\ytrain, \Xtrain} := \E_{\data \sim P^n},~~~~\E_{(\bx, y)} := \E_{(\bx, y) \sim P},~~~~\E_{\bx} := \E_{\bx \sim P_X}.
\]
We define the feature matrix and feature covariance matrix
\[
\bPhi := \left(\phi_p(\bx_1)^\sT, \ldots, \phi_p(\bx_n)^\sT\right) \in \R^{n \times p},~~~~ \bSigma := \E_{\bx}[\phi_p(\bx) \phi_p(\bx)^\sT] \in \R^{p \times p}.
\]
We will also make use of the whitened feature-matrix
\[
\bW := \bPhi \bSigma^{-1/2} \in \R^{n \times p}
\]
and the whitened empirical covariance matrix
\[
\bG := \frac{1}{p}\bW \bW^\sT \in \R^{n \times n}.
\]
We now state our main result
\begin{theorem}[Universal Lower Bound]\label{thm:general_lower_bound}
Let $n, p \geq 1$. Assume that $P$ and $\phi_p$ satisfy the following
\begin{enumerate}
    \item $\E y^2 < \infty$ which implies $\testloss(f_\star) < \infty$,\label{ass:integrable}
    \item $\Var(y \midsmall \bx) \geq \sigma^2$ almost surely over $\bx$,\label{ass:noise}
    \item $\rank(\bPhi) = \min(n, p)$ almost surely. \label{ass:rank}
\end{enumerate}
Then for any algorithm $\cA$ which outputs a linear feature model $\cA(\data)$ with training loss almost surely at most $\ofitlevel \sigma^2$
\[
\E_{\data \sim P^n}~ \excess(\cA(\data)) \geq \sigma^2 \frac{n}{p} \revised{(1 - \sqrt{\tau})^2}.
\]
\end{theorem}
\begin{remark}
Note that we can consider the feature map $\phi_p$ to be random as well, for instance taking $\phi_p$ to be a neural network with output dimension $p$ and random weights $\btheta$. It is often the case that Assumption~\ref{ass:rank} will hold almost surely over $\btheta$ (see Theorem~10 in \cite{holzmuller2020universality}). Since the weights are independent of the data, as an immediate corollary to Theorem \ref{thm:general_lower_bound} we get the same lower bound when additionally taking expectation over $\btheta$. 
\end{remark}
\rev{Now that we have stated our main result and gave some of its interpretations and consequences, we will move on to giving its proof which is pleasantly elementary. We will start by setting up relevant definitions and providing some starting lemmas, before moving on to the core proof.}
\subsection{Proof of Theorem \ref{thm:general_lower_bound}}\label{sec:main_proof}
Under the assumptions of Theorem~\ref{thm:general_lower_bound} we define the minimal excess test loss of $\ofitlevel$-overfitting $p$-dimensional linear feature models trained on the dataset $\data$ of $n$ samples as
\[
    \minexcess(\ofitlevel; n, p) := \min_{\bbeta \in \R^p} \cE(\beta) \st \trainloss(\beta) \leq \ofitlevel \sigma^2.
\]
It then suffices to show that
\[
\E_{\data \sim P^n}[\minexcess(\ofitlevel; n, p)]\geq \sigma^2 \frac{n}{p} \revised{(1 - \sqrt{\ofitlevel})^2}.
\] 
We will bound the excess test loss $\excess(\beta)$ by comparing it with the excess loss with respect to the optimal linear model $\excess^{\lin}(\beta)$. \rev{Let us denote the optimal linear predictor as 
\begin{equation}\label{eq:opt_lin_predictor}
    \bbeta_\star := \argmin_{\bbeta \in \R^p} R(\beta).
\end{equation}
}
\revised{For further analysis we will need the following \rev{two basic} lemmas which are standard results characterizing the excess linear loss}. The proofs can be found in Appendix \ref{app:excess_lin_loss}.
\begin{lemma}[Optimal Linear Predictor]\label{lem:opt_lin}
Define $\bbeta_\star$ as in Eq.\ (\ref{eq:opt_lin_predictor}). Then 
\begin{enumerate}
    \item $\beta_\star$ is the orthogonal projection of $f_\star$ onto the subspace of linear functions $\cH = \{\beta(\bx) : \bbeta \in \R^p\}$ in $L^2$,
    
    \item $\excess(f) := R(f) - R(f_\star) = \E_{\bx}[(f(\bx) - f_\star(\bx))^2]$,
    
    \item $\excess^{\lin}(\beta) := R(\beta) - R(\beta_\star) = \E_{\bx}[(\beta(\bx) - \beta_\star(\bx))^2]$.
\end{enumerate}
\end{lemma}
\rev{The first claim above establishes the equivalent characterization of the optimal linear predictor as a projection of the optimal predictor onto the space of linear functions. The second claim asserts that the excess loss of a given function is its $L^2$ distance to the optimal function, and similarly the third claim asserts that the excess loss a linear function is its $L^2$ distance to the optimal linear function. We now state the second lemma.}
\begin{lemma}[Excess Linear Loss]\label{lem:excess_linear_loss}
The excess loss is lower bounded by the excess linear loss, that is $\excess(\beta) \geq \excess^{\lin}(\beta)$. Moreover we can write the excess linear loss explicitly as
\[
\excess^{\lin}(\beta) = \norm{\bSigma^{1/2}(\bbeta - \bbeta_\star)}^2.
\]
\end{lemma}
\noindent \rev{The results above concern the excess (linear) loss of a given predictor. However to establish lower bounds we will consider the minimal value of this quantity subject to a training loss constraint.} Define the minimal excess linear loss $\auxmin(\ofitlevel)$ for training dataset $\data$ as
\[
\auxmin(\ofitlevel) = \min_{\bbeta \in \R^p} \excess^{\lin}(\beta) \st \trainloss(\beta) \leq \ofitlevel \sigma^2.
\]
Let $\bW = \bPhi \bSigma^{-1/2}$ be the whitened features. Define the random vectors
\[
\noise = (y_i - f_\star(\bx_i))_{i \in [n]} \in \R^n,~~~~ \linnoise = (f_\star(\bx_i) - \beta_\star(\bx_i))_{i \in [n]} \in \R^n
\]
and let $\effnoise = \noise + \linnoise$. \rev{We give an alternate optimization problem for characterizing the minimal excess linear loss which will be more amenable to analysis later on. The same equivalence (for $\tau = 0)$ appears in the proof of Theorem 1 in \cite{muthukumar2020harmless}, however as it is not a very standard result in the literature and is crucial to the rest of our analysis, we record the statement and its proof here.}
\begin{lemma}[Minimal Excess Linear Loss]\label{lem:min_excess_risk_lower_bound}
We can equivalently write $\auxmin(\ofitlevel)$ as
\begin{equation}\label{eq:excess_linear_equiv}
    \auxmin(\ofitlevel) ~=~\min_{\boldb \in \R^p} \norm{\boldb}_2^2~~\mathrm{s.t.}~\frac{1}{n}\norm{\bW \boldb - \effnoise}_2^2 \leq \ofitlevel \sigma^2,
\end{equation}
and $\minexcess(\ofitlevel) \geq \auxmin(\ofitlevel)$.
\end{lemma}
\begin{proof}
By definition we can write the training loss as
\[
\trainloss(\beta) = \frac{1}{n} \norm{\bPhi \bbeta - \ytrain}^2.
\]
Note that $\ytrain = \bPhi \bbeta_\star + \effnoise$, hence
\begin{align*}
    \bPhi \bbeta - \ytrain &= \bPhi \bbeta - (\bPhi \bbeta_\star + \effnoise)\\
    &= \bPhi (\bbeta - \bbeta_\star) - \effnoise\\
    &= \bPhi \bSigma^{-1/2} \bSigma^{1/2} (\bbeta - \bbeta_\star) - \effnoise\\
    &= \bW \bSigma^{1/2} (\bbeta - \bbeta_\star) - \effnoise.
\end{align*}
Therefore by Lemma \ref{lem:excess_linear_loss} we see that $\auxmin(\ofitlevel)$ is equivalent to
\[
 \min_{\bbeta} \norm{\bSigma^{1/2}(\bbeta - \bbeta_\star)}^2 \st \frac{1}{n} \norm{\bW \bSigma^{1/2} (\bbeta - \bbeta_\star) - \effnoise}^2 \leq \ofitlevel \sigma^2,
\]
which after making the following substitution over the optimization variable
\[
\boldb = \bSigma^{1/2}(\bbeta - \bbeta_\star)
\]
can be seen to be equivalent to Eq. (\ref{eq:excess_linear_equiv}). 
\end{proof}
\rev{Before proceeding to the proof of the theorem, we state one more lemma which will be useful for dealing with quadratic forms involving the noise vector. The proof is in Appendix \ref{app:excess_lin_loss}.}
\begin{lemma}[Expectation Over Noise]\label{lem:expected_noise}
Let $f : \R^{n \times d} \to \S^n_+$ be any PSD matrix valued function. Then,
\begin{equation*}
    \E_{\ytrain, \Xtrain} [\effnoise^\sT f(\Xtrain) \effnoise] \geq \sigma^2 \E_{\Xtrain} \Tr(f(\Xtrain)).
\end{equation*}
Moreover, if $\Var(y \midsmall \bx) = \sigma^2$ almost surely and $f_\star = \beta_\star$ then the above is an equality.
\end{lemma}
 We are now ready to give the proof of our main Theorem \ref{thm:general_lower_bound}. \rev{We will proceed to lower bound $\auxmin(\ofitlevel)$ which will then imply the lower bound in Theorem \ref{thm:general_lower_bound} by the previous lemmas. To lower bound $\auxmin(\ofitlevel)$ we will apply weak duality to the optimization problem in Lemma \ref{lem:min_excess_risk_lower_bound}.} Interestingly, as we will see later (see Eq.\ (\ref{eq:lower_bound_by_norm})), it turns out that \rev{this dual lower bound reveals that} the minimal excess linear loss can be bounded below by the \rev{test loss} of a ridge regression estimator \rev{on an auxiliary problem coming from Lemma \ref{lem:min_excess_risk_lower_bound} (see Remark \ref{rmk:ridge}).} 
\begin{proof}[Proof of Theorem \ref{thm:general_lower_bound}]
Consider the optimization problem \revised{in Eq.~(\ref{eq:excess_linear_equiv})} defining $\auxmin(\ofitlevel)$. The associated Lagrangian is
\[
    \cL(\boldb, \lmult) = \norm{\boldb}_2^2 + \lmult\qty(\frac{1}{n} \norm{\bW \boldb - \effnoise}_2^2 - \ofitlevel \sigma^2)
\]
for $\lmult \geq 0$ and the dual function is
\[
g(\lmult) = \inf_{\boldb} \cL(\boldb, \lmult).
\]
After some rearrangement, we can rewrite the Lagrangian as
\[
\cL(\boldb, \lmult) = \boldb^\sT\qty(\id + \frac{\lmult}{n} \bW^\sT \bW)\boldb - \frac{2 \lmult}{n} \boldb^\sT \bW \effnoise + \frac{\lmult}{n}\norm{\effnoise}_2^2 - 
\lmult \ofitlevel \sigma^2
\]
which is a convex quadratic objective in $\boldb$. Hence we can minimize it by setting the derivative to zero. The gradient of $\cL(\boldb, \lmult)$ is given by 
\begin{align*}
    \grad_{\boldb}~\cL(\boldb, \lmult) &= 2\qty(\id + \frac{\lmult}{n} \bW^\sT \bW) \boldb - \frac{2 \lmult}{n} \bW^\sT \effnoise
\end{align*}
hence setting this equal to zero and solving for the optimal $\boldb$ we get
\begin{align*}
    \boldb_\star(\lmult) &= \argmin_{\boldb} \cL(\boldb, \lmult)\\
    &= \lmult(n\id + \lmult \bW^\sT \bW)^{-1}\bW^\sT \effnoise\\
    &= \bW^\sT((n/\lmult)\id + \bW \bW^\sT)^{-1}\effnoise\\
    &= \frac{1}{p}\bW^\sT\qty(\frac{n}{p} \frac{1}{\lmult}\id + \frac{1}{p} \bW \bW^\sT)^{-1}\effnoise.
\end{align*}
By weak duality we have that for any $\lmult \geq 0$,
\[
    \auxmin(\ofitlevel) \geq g(\lmult) = \cL(\boldb_\star(\lmult), \lmult).
\]
For convenience we make the following change of variables 
\[
\ridgeparam := \frac{n}{p} \frac{1}{\lmult}
\]
which is a bijection between $[0, +\infty]$ and itself. Thus in terms of $\ridgeparam$ we can write
\begin{align}
    \cL(\boldb, \ridgeparam) &= \norm{\boldb}^2 + \frac{n}{p\ridgeparam} \qty(\frac{1}{n}\norm{\bW \boldb - \effnoise}^2 - \ofitlevel \sigma^2)\\
    \boldb_\star(\ridgeparam) &= \revised{\frac{1}{p}\bW^\sT\qty(\ridgeparam \id + \frac{1}{p}\bW \bW^\sT)^{-1} \effnoise} \label{eq:lagrangian_opt_b}
\end{align}
\revised{
    where interestingly $\boldb_\star(\lambda)$ happens to be the ridge regression estimator with ridge parameter $\lambda$ on whitened covariates $\bW$ with pure noise target $\effnoise$.} For all $\ridgeparam \geq 0$ we have the lower bound
\[
\auxmin(\ofitlevel) \geq g(\ridgeparam).
\]
Taking expectations, we have the following bound
\begin{equation}\label{eq:expected_dual}
    \E_{\ytrain, \Xtrain} \qty[\auxmin(\ofitlevel)] \geq \E_{\ytrain, \Xtrain}\qty[\cL(\boldb_\star(\ridgeparam), \ridgeparam)] = \E_{\ytrain, \Xtrain} \norm{\boldb_\star(\ridgeparam)}^2 + \frac{n}{p \ridgeparam} \E_{\ytrain, \Xtrain} \qty(\frac{1}{n}\norm{\bW \boldb_\star(\ridgeparam) - \effnoise}_2^2 - \ofitlevel \sigma^2).
\end{equation}
Let $\bG = \bW\bW^{\sT}/p$ and $\bG(\ridgeparam) = \bG + \ridgeparam \id$. Denote the eigenvalues of $\bG$ as $\lambda_1 \geq \lambda_2 \geq \ldots \geq \lambda_n$. We will show
\begin{align}
    \E_{\ytrain, \Xtrain} \norm{\boldb_\star(\ridgeparam)}^2 \geq \frac{\sigma^2}{p} \E_{\Xtrain} \Tr(\bG(\ridgeparam)^{-1} \bG \bG(\ridgeparam)^{-1}) &= \frac{\sigma^2}{p} \E_{\Xtrain} \sum\limits_{i = 1}^n \frac{\lambda_i}{(\lambda_i + \ridgeparam)^2} \label{eq:expected-norm}\\
    \E_{\ytrain, \Xtrain} \frac{1}{n}\norm{\bW \boldb_\star(\ridgeparam) - \effnoise}_2^2 \geq \frac{\sigma^2}{n} \E_{\Xtrain} \Tr(\bG \bG(\ridgeparam)^{-1} - \id)^2 &= \frac{\sigma^2}{n} \E_{\Xtrain} \sum\limits_{i = 1}^n \qty(\frac{\ridgeparam}{\lambda_i + \ridgeparam})^2 \label{eq:expected-constraint}
\end{align}
using Lemma \ref{lem:expected_noise}. For Eq. (\ref{eq:expected-norm}) we have the following
\begin{align*}
    \E_{\ytrain, \Xtrain} \norm{\boldb_\star(\ridgeparam)}^2 &= \frac{1}{p} \E_{\ytrain, \Xtrain}\qty[\effnoise^\sT \bG(\ridgeparam)^{-1} \bG \bG(\ridgeparam)^{-1} \effnoise]\\
    &\geq \frac{\sigma^2}{p} \E_{\Xtrain} \Tr(\bG(\ridgeparam)^{-1} \bG \bG(\ridgeparam)^{-1}).
\end{align*}
Similarly, for Eq. (\ref{eq:expected-constraint})
\begin{align*}
    \E_{\ytrain, \Xtrain} \frac{1}{n}\norm{\bW \boldb_\star(\ridgeparam) - \effnoise}_2^2 &= \frac{1}{n}\E_{\ytrain, \Xtrain} [\effnoise^\sT(\bG \bG(\ridgeparam)^{-1} - \id)^2\effnoise] \\
    &\geq \frac{\sigma^2}{n} \E_{\Xtrain} \Tr(\bG \bG(\ridgeparam)^{-1} - \id)^2.
\end{align*}
Note that by the full rank assumption, $\lambda_1 \geq \cdots \geq \lambda_{\min(n, p)} > 0$ and $\lambda_{\min(n, p) + 1} = \ldots = \lambda_n = 0$ almost surely. Define the function $f : [0, +\infty] \to \R$ as
\[
f(\ridgeparam) := \frac{1}{n}\sum\limits_{i = 1}^n \qty(\frac{\ridgeparam}{\lambda_i + \ridgeparam})^2 = \frac{1}{n}\sum\limits_{i = 1}^{\min(n, p)} \qty(\frac{\ridgeparam }{ \lambda_i + \ridgeparam })^2 + \max(0, 1 - p/n).
\]
Note that $f$ is continuous and $\inf_\ridgeparam f(\ridgeparam) = f(0) = \max(0, 1 - p/n)$ and $\sup_\ridgeparam f(\ridgeparam) = f(\infty) = 1$. Therefore as long as $1 - p/n \leq \ofitlevel$ there exists $\ridgeparam_\star$ such that 
\begin{equation}\label{eq:lagrange_multiplier_fixed_point}
    f(\ridgeparam_\star) = \frac{1}{n} \sum\limits_{i = 1}^n \qty(\frac{\ridgeparam_\star}{\lambda_i + \ridgeparam_\star})^2 = \ofitlevel.
\end{equation}                                             
Otherwise if \revised{$1 - p/n > \ofitlevel$}, then by taking $\ridgeparam \to 0$, we have that 
\[
   \frac{1}{\ridgeparam} \cdot \E_{\ytrain, \Xtrain} \qty(\frac{1}{n}\norm{\bW \boldb_\star(\ridgeparam) - \effnoise}_2^2 - \ofitlevel \sigma^2) \to \infty,
\]
which by inequality Eq. (\ref{eq:expected_dual}) implies $\E_{\ytrain, \Xtrain}\qty[\auxmin(\ofitlevel)] = +\infty$ and any lower bound holds trivially. Thus we will assume that \revised{$p/n \geq 1 - \tau$, in which case we will be able to obtain a non-vacuous result}. Let $\ridgeparam_\star$ be the random variable dependent on $\Xtrain$ that satisfies Eq. (\ref{eq:lagrange_multiplier_fixed_point}). Then by Eq. (\ref{eq:expected-constraint}) we have that
\[
\E_{\by, \bX}\qty(\frac{1}{n}\norm{\bW \boldb_\star(\ridgeparam_\star) - \effnoise}_2^2 - \ofitlevel \sigma^2) \geq 0.
\]
Thus from Eq. (\ref{eq:expected_dual}) we have that
\begin{align}
    \E_{\ytrain, \Xtrain}\qty[\auxmin(\ofitlevel)] &\geq \E_{\ytrain, \Xtrain} \norm{\boldb_\star(\ridgeparam_\star)}^2 + \frac{n}{p \ridgeparam_\star} \E_{\ytrain, \Xtrain} \qty(\frac{1}{n}\norm{\bW \boldb_\star(\ridgeparam_\star) - \effnoise}_2^2 - \ofitlevel \sigma^2)\nonumber\\ 
    &\geq \E_{\ytrain, \Xtrain} \norm{\boldb_\star(\ridgeparam_\star)}^2.\label{eq:lower_bound_by_norm}
\end{align}
Hence from Eq. (\ref{eq:expected-norm}) we have
\begin{equation}\label{eq:excess-lb-opt-lagrange}
    \E_{\ytrain, \Xtrain}\qty[\auxmin(\ofitlevel)] \geq \sigma^2\frac{n}{p} \cdot \E_{\Xtrain}~ \frac{1}{n} \sum\limits_{i = 1}^n \frac{\lambda_i}{(\lambda_i + \ridgeparam_\star)^2}.
\end{equation}

Thus to lower bound $\E[\auxmin(\tau)]$ we can try to lower bound the following
\begin{equation}\label{eq:eigenvalue_objective}
    \frac{1}{n} \sum\limits_{i = 1}^n \frac{\lambda_i}{(\lambda_i + \ridgeparam_\star)^2}~~ \text{ subject to } ~~\frac{1}{n} \sum\limits_{i=1}^n \qty(\frac{\lambda_\star}{\lambda_i + \lambda_\star})^2 = \tau,
\end{equation}
by a quantity that we can later easily bound in expectation over $\Xtrain$. Note that if $\lambda_1 = \ldots = \lambda_n = \lambda$, then 
\begin{equation}\label{eq:equal_eigenvalues}
    \lambda_\star = \lambda \frac{\sqrt{\tau}}{(1 - \sqrt{\tau})}~~ \text{ and } ~~\frac{1}{n} \sum\limits_{i = 1}^n \frac{\lambda_i}{(\lambda_i + \ridgeparam_\star)^2} = \frac{1}{\lambda}(1 - \sqrt{\tau})^2
\end{equation}
and if $\lambda_\star = 0$ then
\begin{equation}\label{eq:lambda_star_zero}
    \sum\limits_{i = 1}^n \frac{\lambda_i}{(\lambda_i + \ridgeparam_\star)^2} = \frac{1}{n} \sum\limits_{i=1}^n \frac{1}{\lambda_i} \geq \frac{n}{\sum\limits_{i=1}^n \lambda_i}
\end{equation}
where the last inequality is the AM-HM inequality (Lemma \ref{lem:AM-HM}). Therefore a possible lower bound for Eq.\ (\ref{eq:eigenvalue_objective}) is 
\begin{equation}\label{eq:deterministic_lower_bound}
    \frac{1}{n} \sum\limits_{i = 1}^n \frac{\lambda_i}{(\lambda_i + \ridgeparam_\star)^2} \geq \frac{n}{\sum\limits_{i = 1}^n \lambda_i} (1 - \sqrt{\tau})^2,
\end{equation}
as it satisfies the edge cases in Eqs.\ (\ref{eq:equal_eigenvalues}) and (\ref{eq:lambda_star_zero}). We will show that this inequality in fact holds.

By Eq.\ (\ref{eq:lagrange_multiplier_fixed_point}), to prove that Eq.\ (\ref{eq:deterministic_lower_bound}) holds, it suffices to show that for any $a_1 \geq \ldots \geq a_n \geq 0$ and $x \geq 0$,
\[
\frac{1}{n} \sum\limits_{i = 1}^n \frac{a_i}{(a_i + x)^2} \geq \frac{n}{\sum\limits_{i = 1}^n a_i} \qty(1 - \sqrt{\frac{1}{n} \sum\limits_{i=1}^n \qty(\frac{x}{a_i + x})^2})^2
\]
which upon rearranging is equivalent to
\begin{equation}\label{eq:equiv_inequality}
    \sqrt{\frac{1}{n} \sum\limits_{i = 1}^n \frac{a_i}{(a_i + x)^2}} \sqrt{\frac{1}{n}\sum\limits_{i = 1}^n a_i} + \sqrt{\frac{1}{n} \sum\limits_{i=1}^n \qty(\frac{x}{a_i + x})^2} \geq 1.
\end{equation}
Let us fix $a_1 \geq \ldots \geq a_n \geq 0$ and consider the function $g : [0, \infty) \to [0, \infty)$ defined as the left-hand side of the above
\begin{equation}\label{eq:equiv_ineq_function}
    g(x) = \sqrt{\frac{1}{n} \sum\limits_{i = 1}^n \frac{a_i}{(a_i + x)^2}} \sqrt{\frac{1}{n}\sum\limits_{i = 1}^n a_i} + \sqrt{\frac{1}{n} \sum\limits_{i=1}^n \qty(\frac{x}{a_i + x})^2}.
\end{equation}
We can show $g(x)$ is decreasing by computing the derivative which is given as follows
\[
g'(x) = \frac{1}{n} \sum\limits_{i=1}^n \frac{a_i}{(a_i + x)^3} \cdot \qty[\qty(\frac{1}{n} \sum\limits_{i = 1}^n \frac{1}{(a_i + x)^2})^{-1/2} - \qty(\frac{1}{n} \sum\limits_{i=1}^n a_i)^{1/2} \cdot \qty(\frac{1}{n} \sum\limits_{i = 1}^n \frac{a_i}{(a_i + x)^2})^{-1/2} ].
\]
Therefore to show that $g'(x) \leq 0$, by rearranging the above expression it suffices show that
\[
\frac{1}{n} \sum\limits_{i=1}^n \frac{a_i}{(a_i + x)^2} \leq \qty(\frac{1}{n} \sum\limits_{i=1}^n a_i) \cdot \qty(\frac{1}{n} \sum\limits_{i=1}^n \frac{1}{(x + a_i)^2}).
\]
The above however follows immediately from Chebyshev's Sum Inequality (Lemma \ref{lem:chebyshev-sum-inequality}). Therefore we have shown that $g(x)$ is decreasing. Since in Eq.\ (\ref{eq:equiv_ineq_function}) it is easy to see that $\lim\limits_{x \to \infty} g(x) = 1$, it follows that $g(x) \geq 1$ for all $x \geq 0$, which proves Eq.\ (\ref{eq:equiv_inequality}). Now taking Eq.\ (\ref{eq:equiv_inequality}) and plugging in to Eq.\ (\ref{eq:excess-lb-opt-lagrange}) we get that
\begin{align*}
    \E_{\ytrain, \Xtrain}~ \auxmin(\ofitlevel)  &\geq \sigma^2\frac{n}{p} \cdot \E_{\Xtrain}~ \frac{1}{n} \sum\limits_{i = 1}^n \frac{\lambda_i}{(\lambda_i + \ridgeparam_\star)^2}\\
    &\geq \sigma^2\frac{n}{p} \cdot  \E_{\Xtrain}~\frac{n}{\sum\limits_{i = 1}^n \lambda_i} (1 - \sqrt{\tau})^2\\
    &\geq \sigma^2\frac{n}{p} (1 - \sqrt{\tau})^2 \frac{n}{\E_{\Xtrain} \sum\limits_{i = 1}^n \lambda_i} && (\text{Jensen's Inequality})\\
    &= \sigma^2\frac{n}{p} (1 - \sqrt{\tau})^2,
\end{align*}
where the last equality holds since $\E_{\Xtrain} \sum\limits_{i = 1}^n \lambda_i = \E_{\Xtrain} \Tr(\bG) = n$.
\end{proof}
\begin{remark}
One may wonder if the following lower bound
\begin{equation}\label{eq:invalid_lower_bound}
    \frac{1}{n} \sum\limits_{i = 1}^n \frac{\lambda_i}{(\lambda_i + \ridgeparam_\star)^2} \geq \frac{1}{n} \sum\limits_{i=1}^n \frac{1}{\lambda_i} (1 - \sqrt{\tau})^2,
\end{equation}
could have been used in place of Eq.\ (\ref{eq:deterministic_lower_bound}) as it also satisfies the same edge cases. This however is not a valid inequality. To see this, take $n = 2$ and let $\lambda_1 = \varepsilon$, $\lambda_2 = 1 - \varepsilon$, and $\tau \in (1/2, 1)$. Then as $\varepsilon \to 0$, it is easy to see since 
\[
\tau = \frac{1}{n} \sum\limits_{i=1}^n \qty(\frac{\lambda_\star}{\lambda_i + \lambda_\star})^2 \approx \frac{1}{2} + \frac{1}{2} \qty(\frac{\lambda_\star}{1 + \lambda_\star})^2
\]
it must be that $\lambda_\star$ is bounded below and does not go to 0. However that means that the left-hand side of Eq.\ (\ref{eq:invalid_lower_bound}) is bounded above whereas the right hand side goes to infinity.    
\end{remark}

\begin{remark}\label{rmk:ridge}
\rev{The lower bound via weak duality (see Eq.\ (\ref{eq:lower_bound_by_norm})) reveals that the minimal excess linear loss is lower bounded by the test loss of a ridge estimator on a different, auxiliary  problem coming from Lemma \ref{lem:min_excess_risk_lower_bound}. 
In this auxiliary problem the covariates are whitened and the training targets are the noise $\effnoise$ as opposed to $\ytrain$. The ridge parameter of the estimator is exactly the Lagrange multiplier corresponding to the training error constraint. This ridge estimator requires oracle knowledge of the feature covariance $\bSigma$ and the noise vector $\effnoise$. Note that the lower bound arises purely from the variance due to label noise since this oracle ridge estimator is an unbiased estimator for the original problem. 
}
\end{remark}

%% file: sections/asymptotic.tex
\section{Lower Bounds under Marchenko-Pastur Asymptotics}\label{sec:mp}
In this section we analyze $\minexcess(\ofitlevel; n, p)$  more precisely under additional asymptotic distributional assumptions. \rev{As mentioned earlier, this will allow us to assess the tightness of our general lower bound given in Theorem \ref{thm:general_lower_bound}}. Specifically,
we assume the following setting which has been used to analyse ridge regression in several prior works including \cite{dicker2016ridge, dobriban2018high, hastie2022surprises}:
   \paragraph{[MP]} Assume $n, p \to \infty$ so that $n/p \to \gamma \in (0, +\infty)$. Recall the matrix $\bG = \bW \bW^\sT / p$ which has eigenvalues $\ridgeparam_1, \ldots, \ridgeparam_n$. Let $\mu_{p}$ be the empirical spectral distribution of $\bG$,
     \[
    \mu_{p}(A) = \frac{1}{n} \qty|\{\lambda_i \in A\}|,~~~~A \subset \R.
    \]
    Then we have the following convergence in the weak topology
    \[
        \mu_{p} \to \MP(\gamma)
    \]
    where $\MP(\gamma)$ is the Marchenko-Pastur distribution with aspect ratio $\gamma$ (see Appendix \ref{sec:app-mp}). 
    
\paragraph{[Lin]} The optimal model is linear and noise is additive: 
\[
    y = \bbeta_\star^\sT \phi(\bx) + \varepsilon,
\] 
where the noise $\varepsilon$ satisfies $\E[\varepsilon] = 0$, $\E[\varepsilon^2] = \sigma^2$, and $\E[\varepsilon^{4 + \eta}] < \infty$ for some $\eta > 0$.
\begin{remark}
Assumption \textbf{MP} holds whenever the entries of ~$\bW$ are distributed i.i.d with mean zero and variance $1$. In particular, this holds if $\phi(\bx) \sim \cN(0, \bSigma)$ for any $\bSigma \succ \bzero$. By assuming the optimal model is linear we have that $\minexcess(\ofitlevel) = \auxmin(\ofitlevel)$. Additionally, since $\Var(y \midsmall \bx) = \Var(\varepsilon) = \sigma^2$ almost surely, the inequality in Lemma \ref{lem:expected_noise} is an equality.
\end{remark}
Let $m(z; \gamma)$ denote the Stieltjes transform of the Marchenko-Pastur Law and let $m'(z; \gamma)$ be the derivative with respect to $z$ (see Appendix \ref{sec:app-mp}). Under the above assumptions we have the following analytical characterization of the asymptotic minimal excess error.
\begin{proposition}[MP Asymptotics]\label{prop:mp-asymp}
Under Assumptions \textbf{MP} and \textbf{Lin}, the asymptotic minimal excess error is given by the following analytical expression
\begin{equation}\label{eq:min_excess_loss_analytic}
     \minexcess(\ofitlevel, \gamma) := \lim\limits_{\substack{n, p \to \infty\\ n/p \to \gamma}} \minexcess(\ofitlevel; n, p) = \sigma^2 \gamma [m(-\ridgeparam; \gamma) - \ridgeparam m'(-\ridgeparam; \gamma)]
\end{equation}
where $\ridgeparam$ satisfies the fixed point equation
\begin{equation}\label{eq:asymp_fixed_point}
    \ridgeparam^2 m'(-\ridgeparam; \gamma) = \ofitlevel.
\end{equation}
\end{proposition}
\begin{proof}
For any $\ridgeparam \geq 0$, by Assumption \textbf{Lin} we can use Proposition \ref{prop:convergence_quad_form} to get the asymptotic versions of Eqs.\ (\ref{eq:expected-norm}), (\ref{eq:expected-constraint})
\begin{align*}
    \norm{\boldb_\star(\ridgeparam)}^2 &\sim \sigma^2 \frac{1}{p} \Tr(\bG(\ridgeparam)^{-1} \bG \bG(\ridgeparam)^{-1})\\
    \frac{1}{n} \norm{\bW \boldb_\star(\ridgeparam) - \effnoise}^2 &\sim \sigma^2\frac{1}{n} \Tr(\bG \bG(\ridgeparam)^{-1} - \id)^2
\end{align*}
where $a_n \sim b_n$ denotes that  $|a_n - b_n| \to 0$ almost surely. Then by Assumption \textbf{MP}
\begin{align*}
\frac{1}{p} \Tr(\bG(\ridgeparam)^{-1} \bG \bG(\ridgeparam)^{-1}) &= \frac{1}{p} \sum\limits_{i=1}^n \frac{\lambda_i}{(\lambda_i + \ridgeparam)^2}
    =  \frac{n}{p} \int \frac{s}{(s + \ridgeparam)^2} \dd{\mu_p(s)} \to \gamma \int \frac{s}{(s + \ridgeparam)^2} \dd{\MP_\gamma(s)},\\
    \frac{1}{n} \Tr(\bG \bG(\ridgeparam)^{-1} - \id)^2 &= \frac{1}{n} \sum\limits_{i=1}^n \qty(\frac{\ridgeparam}{\lambda_i + \ridgeparam})^2 = \int \qty(\frac{\ridgeparam}{s + \ridgeparam})^2 \dd{\mu_p(s)} \to \int \qty(\frac{\ridgeparam}{s + \ridgeparam})^2 \dd{\MP_\gamma(s)}.
\end{align*}
Therefore if $\ridgeparam_\star$ is the unique solution to the fixed point equation
\begin{equation}\label{eq:asymp_fixed_point_integral}
    \int \qty(\frac{\ridgeparam_\star}{s + \ridgeparam_\star})^2 \dd{\MP_\gamma(s)} = \ofitlevel 
\end{equation}
then the constraint becomes tight almost surely
\[
\frac{1}{n} \norm{\bW \boldb_\star(\ridgeparam_\star) - \effnoise}^2 - \ofitlevel \sigma^2 \to 0.
\]
The pair $(\boldb(\ridgeparam_\star), \ridgeparam_\star)$ is the unique KKT point and since a minimizer of the primal problem exists due to continuity of the objective and compactness of the constraint set, this pair is asymptotically the primal/dual optimal variables and
\[
\minexcess(\ofitlevel; n, p) \sim \norm{\boldb_\star(\ridgeparam_\star)}^2 \to \sigma^2 \gamma \int \frac{s}{(s + \ridgeparam_\star)^2} \dd{\MP_\gamma(s)}
\]
where $\ridgeparam_\star$ satisfies Eq.\ (\ref{eq:asymp_fixed_point_integral}). We can write the integrals that appear above in terms of the Stieltjes transform
\begin{align*}
    \int \frac{s}{(s + z)^2} \dd{\MP_\gamma(s)} &= \int \qty[\frac{1}{s + z}  - \frac{z}{(s + z)^2}]\dd{\MP_\gamma(s)}
    = m(-z; \gamma) - z m'(-z; \gamma), \\
    \int \qty(\frac{z}{s + z})^2 \dd{\MP_\gamma(s)} &= z^2 \int \frac{1}{(s + z)^2} \dd{\MP_\gamma(s)} = z^2 m'(-z; \gamma).
\end{align*}
Thus we can write $\minexcess(\ofitlevel, \gamma)$ as
\[
\sigma^2 \gamma [m(-\ridgeparam; \gamma) - \ridgeparam m'(-\ridgeparam; \gamma)] \st \ridgeparam^2 m'(-\ridgeparam; \gamma) = \ofitlevel.
\]
\end{proof}
\begin{remark}
    Note that by Remark \ref{rmk:ridge}, $\cE_\star(\tau, \gamma)$ should be equal to the asymptotic test risk of a ridge regression predictor on isotropic Gaussian covariates with ridge parameter $\lambda$ satisfying Eq.\ (\ref{eq:asymp_fixed_point}), when the target function is zero. Indeed our calculations match results obtained in previous calculations of this limiting risk, for example Corollary 5 in \cite{hastie2022surprises}.
\end{remark}
\noindent For convenience in later proofs, let use define the following functions
\begin{align}
    E(\ridgeparam, \gamma) &= m(-\ridgeparam; \gamma) - \ridgeparam m'(-\ridgeparam; \gamma),\label{eq:func_E}\\ 
    f(\ridgeparam, \gamma) &= \ridgeparam^2 m'(-\ridgeparam; \gamma). \label{eq:func_f}
\end{align}
Note that $f$ is strictly increasing in $\lambda$, so we can define the inverse function $f^{-1}(\ofitlevel, \gamma)$ so that 
\begin{equation}\label{eq:inverse}
    f(f^{-1}(\ofitlevel, \gamma), \gamma) = \ofitlevel.
\end{equation}
Using our characterization of the minimal excess error in Proposition \ref{prop:mp-asymp} we will now derive lower bounds in Theorems \ref{thm:mp-law-lower-bound} and \ref{thm:mp-law-local} and an exact expression when $\gamma = 1$ in Theorem \ref{thm:peak}.
\begin{figure}\label{fig:mp-lower-bound}
    \centering
    \begin{subfigure}[t]{0.45\textwidth}
        \centering
        \includegraphics[width=\linewidth]{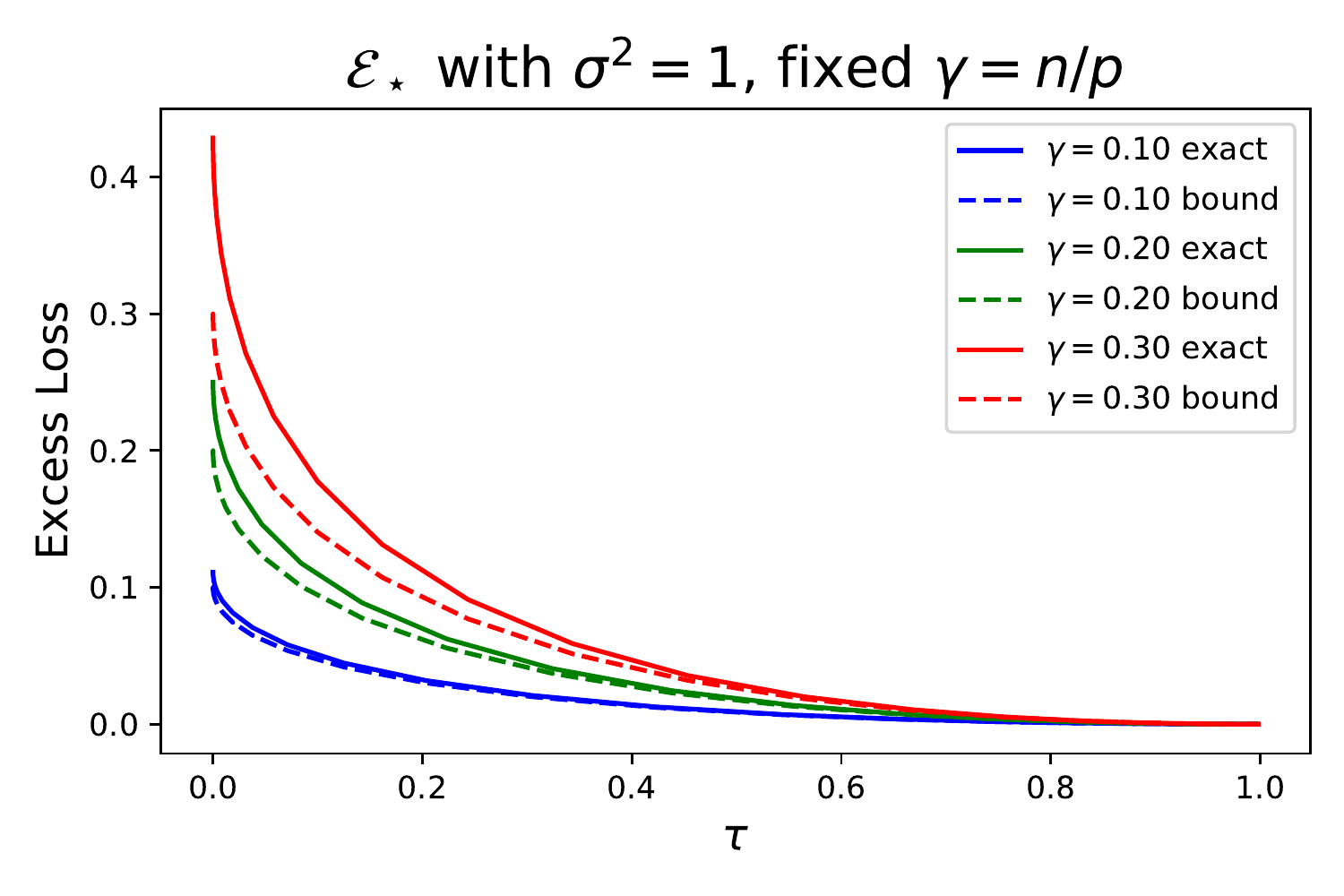}
        \caption{For fixed $\gamma = n/p$ we plot the exact value of $\minexcess(\ofitlevel)$ using Eq.\ (\ref{eq:min_excess_loss_analytic}) [solid] and the lower bound from Eq.\ (\ref{eq:mp-lb}) [dashed].}
    \end{subfigure}
    \begin{subfigure}[t]{0.45\textwidth}
        \centering
        \includegraphics[width=\linewidth]{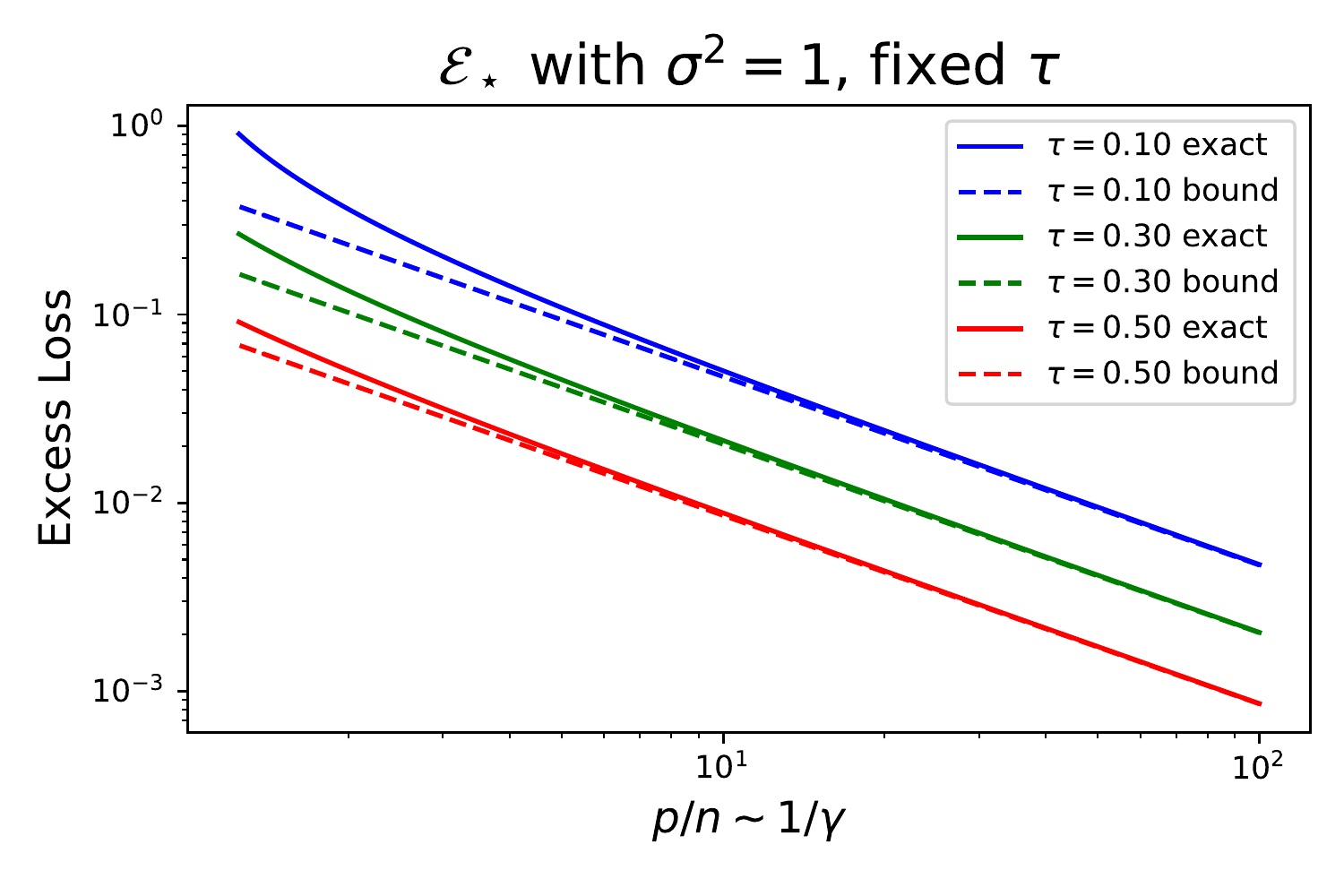}
        \caption{For fixed $\ofitlevel$ we plot the exact value of $\minexcess(1/\gamma)$ using Eq.\ (\ref{eq:min_excess_loss_analytic}) [solid] and the lower bound from Eq.\ (\ref{eq:mp-lb}) [dashed].}
    \end{subfigure}
\end{figure}

\begin{theorem}[MP Lower Bound]\label{thm:mp-law-lower-bound}
\revised{By Theorem \ref{thm:general_lower_bound}, for} fixed $\ofitlevel \in [0, 1]$, the minimum excess error satisfies
\begin{equation}\label{eq:mp-lb}
\minexcess(\ofitlevel, \gamma) \geq \sigma^2 \gamma (1 - \sqrt{\ofitlevel})^2
\end{equation}
for all $\gamma \in (0, +\infty)$. The dependence on $\ofitlevel$ is \revised{in fact} tight since 
\begin{equation}\label{eq:mp-lb-tight}
    \inf_{\gamma \in (0, \infty)} \minexcess(\ofitlevel, \gamma) / \gamma = \sigma^2 (1 - \sqrt{\ofitlevel})^2.
\end{equation}
\end{theorem}
\remove{Note that Eq.\ (\ref{eq:mp-lb-tight}) directly implies Eq.\ (\ref{eq:mp-lb}), hence we will just show that Eq.\ (\ref{eq:mp-lb-tight}) holds. To show Eq.\ (\ref{eq:mp-lb-tight}), we will prove in Lemma \ref{thm:decreasing_error} that the quantity $\minexcess(\ofitlevel, \gamma) / \gamma$ is increasing in $\gamma$, hence the infimum is given by the limit as $\gamma \to 0$ which turns out to be simple to compute.} 
\begin{proof}
    Observe that as $\gamma \to 0$, the spectral eigenvalue distribution of \revised{$\bW \bW^\sT / p$} is converging to a point mass at 1. Therefore, \revised{recalling the function $f$ defined in Eq.\ (\ref{eq:func_f}),} by the Dominated Convergence Theorem
    \[
    \lim_{\gamma \to 0} f(\ridgeparam, \gamma) =\lim_{\gamma \to 0} \int_{s \in [1 - \sqrt{\gamma}, 1 + \sqrt{\gamma}]} \frac{\ridgeparam^2}{(s + \ridgeparam)^2} \dd{H}_\gamma(s) = \frac{\ridgeparam^2}{(1 + \ridgeparam)^2}
    \]
    hence by continuity
    \[
    \lim_{\gamma \to 0} f^{-1}(\ofitlevel, \gamma) = \frac{\revised{\sqrt{\ofitlevel}}}{(1 - \revised{\sqrt{\ofitlevel}})}.
    \]
    Thus letting $\ridgeparam =  \lim_{\gamma \to 0} f^{-1}(\ofitlevel, \gamma)$, we get
    \[
    \lim_{\gamma \to 0} \frac{\minexcess(\ofitlevel, \gamma)}{\sigma^2 \gamma} = \lim_{\gamma \to 0} \int_{s \in [1 - \sqrt{\gamma}, 1 + \sqrt{\gamma}]} \frac{s}{(s + \ridgeparam)^2} \dd{H}_\gamma(s) = \frac{1}{(1 + \ridgeparam)^2} = (1 - \revised{\sqrt{\ofitlevel}})^2.
    \]
    \revised{By Eq.\ (\ref{eq:mp-lb})}, we have that the limit equals the infimum 
    \[
    \lim_{\gamma \to 0} \frac{\minexcess(\ofitlevel, \gamma)}{\sigma^2 \gamma} = \inf_{\gamma \in (0, \infty)} \frac{\minexcess(\ofitlevel, \gamma)}{\sigma^2 \gamma}
    \]
    which shows Eq.\ (\ref{eq:mp-lb-tight})\remove{, from which Eq.\ (\ref{eq:mp-lb}) is a direct consequence}.
\end{proof}
\revised{From the above Theorem \ref{thm:mp-law-lower-bound} we saw that $\minexcess(\ofitlevel, \gamma) / \gamma \to (1 - \sqrt{\tau})^2$ as $\gamma \to 0$, achieving its infimum in the limit. In the following theorem we will show that moreover $\minexcess(\ofitlevel, \gamma) / \gamma$ is strictly increasing in $\gamma$ which implies in particular that Eq.\ (\ref{eq:mp-lb}) becomes strictly looser as $\gamma$ grows.}

\begin{theorem}\label{thm:decreasing_error}
For \revised{fixed} $\ofitlevel \in (0, 1]$, \revised{the ratio of the minimum excess error $\minexcess(\ofitlevel, \gamma)$ to $\gamma$ is increasing in $\gamma$, that is}
\[
    \revised{\dv{\minexcess(\ofitlevel, \gamma) / \gamma}{\gamma} = \sigma^2} \dv{E(f^{-1}(\ofitlevel, \gamma), \gamma)}{\gamma}\; \revised{ > } \;0,\; \revised{\gamma \in (0, +\infty)}
    \]
\revised{where the function $E$ is defined in Eq.\ (\ref{eq:func_E}).}
\end{theorem}
\begin{proof}
Let $\ridgeparam = f^{-1}(\ofitlevel, \gamma)$. By the chain rule
\begin{equation}\label{eq:chain-rule}
    \dv{E(f^{-1}(\ofitlevel, \gamma), \gamma)}{\gamma} = \pdv{E(\ridgeparam, \gamma)}{\gamma} + \pdv{E(\ridgeparam, \gamma)}{\ridgeparam} \pdv{f^{-1}(\ofitlevel, \gamma)}{\gamma}.
\end{equation}
We will compute the following three terms in Eq.\ (\ref{eq:chain-rule})
\[
(1):~~ \pdv{E(\ridgeparam, \gamma)}{\gamma},~~~~(2):~~ \pdv{E(\ridgeparam, \gamma)}{\ridgeparam},~~~~(3):~~\pdv{f^{-1}(\ofitlevel, \gamma)}{\gamma}.
\]
The computation of the first two terms is direct
\begin{align*}
    \pdv{E(\ridgeparam, \gamma)}{\gamma} &= \pdv{m(-\ridgeparam; \gamma)}{\gamma} - \ridgeparam \pdv{m'(-\ridgeparam; \gamma)}{\gamma}, \\
    \pdv{E(\ridgeparam, \gamma)}{\ridgeparam} &= -2m'(-\ridgeparam; \gamma) + \ridgeparam m''(-\ridgeparam; \gamma).
\end{align*}
For the third term, we take the derivative with respect to $\gamma$ on both sides of Eq.\ (\ref{eq:inverse}) which yields
\[
\pdv{f(\ridgeparam, \gamma)}{\ridgeparam} \pdv{f^{-1}(\ofitlevel, \gamma)}{\gamma}  + \pdv{f(\ridgeparam, \gamma)}{\gamma} = 0.
\]
Hence after re-arranging we have
\[
    \pdv{f^{-1}(\ofitlevel, \gamma)}{\gamma} = -\frac{\pdv{f(\ridgeparam, \gamma)}{\gamma}}{\pdv{f(\ridgeparam, \gamma)}{\ridgeparam}} = - \frac{\ridgeparam^2 \pdv{m'(-\ridgeparam; \gamma)}{\gamma}}{2 \ridgeparam m'(-\ridgeparam; \gamma) - \ridgeparam^2 m''(-\ridgeparam; \gamma)}.
\]
Computing the product of terms (2) and (3) gives
\[
\pdv{E(\ridgeparam, \gamma)}{\ridgeparam}\pdv{f^{-1}(\ofitlevel, \gamma)}{\gamma} = \ridgeparam \pdv{m'(-\ridgeparam; \gamma)}{\gamma}.
\]
Finally putting everything together in Eq.\ (\ref{eq:chain-rule}) yields
\begin{align*}
    \dv{E(f^{-1}(\ofitlevel, \gamma), \gamma)}{\gamma} &= \pdv{E(\ridgeparam, \gamma)}{\gamma} + \pdv{E(\ridgeparam, \gamma)}{\ridgeparam} \pdv{f^{-1}(\ofitlevel, \gamma)}{\gamma}\\
    &= \pdv{m(-\ridgeparam; \gamma)}{\gamma} - \ridgeparam \pdv{m'(-\ridgeparam; \gamma)}{\gamma} + \ridgeparam \pdv{m'(-\ridgeparam; \gamma)}{\gamma}\\
    &= \pdv{m(-\ridgeparam; \gamma)}{\gamma}\Bigg|_{\ridgeparam = f^{-1}(\ofitlevel, \gamma)} \, \revised{ > } \, 0,
\end{align*}
where the last inequality follows from Lemma \ref{lem:stieltjes_gamma_increasing}.
\end{proof}

The previous bound provides a lower bound on the minimum excess risk that holds for all $\ofitlevel \in [0, 1]$ and becomes tight as $\gamma \to 0$ i.e. the amount of overparametrization becomes large. However, the bound is loose for small $\ofitlevel$ near the interpolation peak $\gamma = 1$. To understand behavior in this regime, for a fixed $\gamma$ we compute a local expansion of $\minexcess(\ofitlevel, \gamma)$ around $\minexcess(0, \gamma)$. Instead of directly analyzing the function $\minexcess(\ofitlevel, \gamma)$, we work with the function $\sqrt{\minexcess(t, \gamma)}$ where we re-parametrize in terms of $t := \sqrt{\ofitlevel}$. We then obtain the bound by taking the first-order Taylor expansion of this function around $t = 0$ and then arguing it is a lower bound by showing the function is convex. In the end we convert this into lower bound on the original function $\minexcess(\ofitlevel, \gamma)$. It is important to re-parametrize in terms of $\sqrt{\ofitlevel}$ since the derivative of $\minexcess(\ofitlevel, \gamma)$ with respect to $\ofitlevel$ at $0$ is $-\infty$, as is also true for the lower bound in Eq.\ (\ref{eq:mp-lb}). Intriguingly, as we explain in Remark \ref{rmk:square-root-sharp}, taking the Taylor expansion of the square-root of $\minexcess$ yields a tighter approximation rather than directly expanding $\minexcess$. 
\begin{figure}[t]
    \centering
    \includegraphics[width=0.8\linewidth]{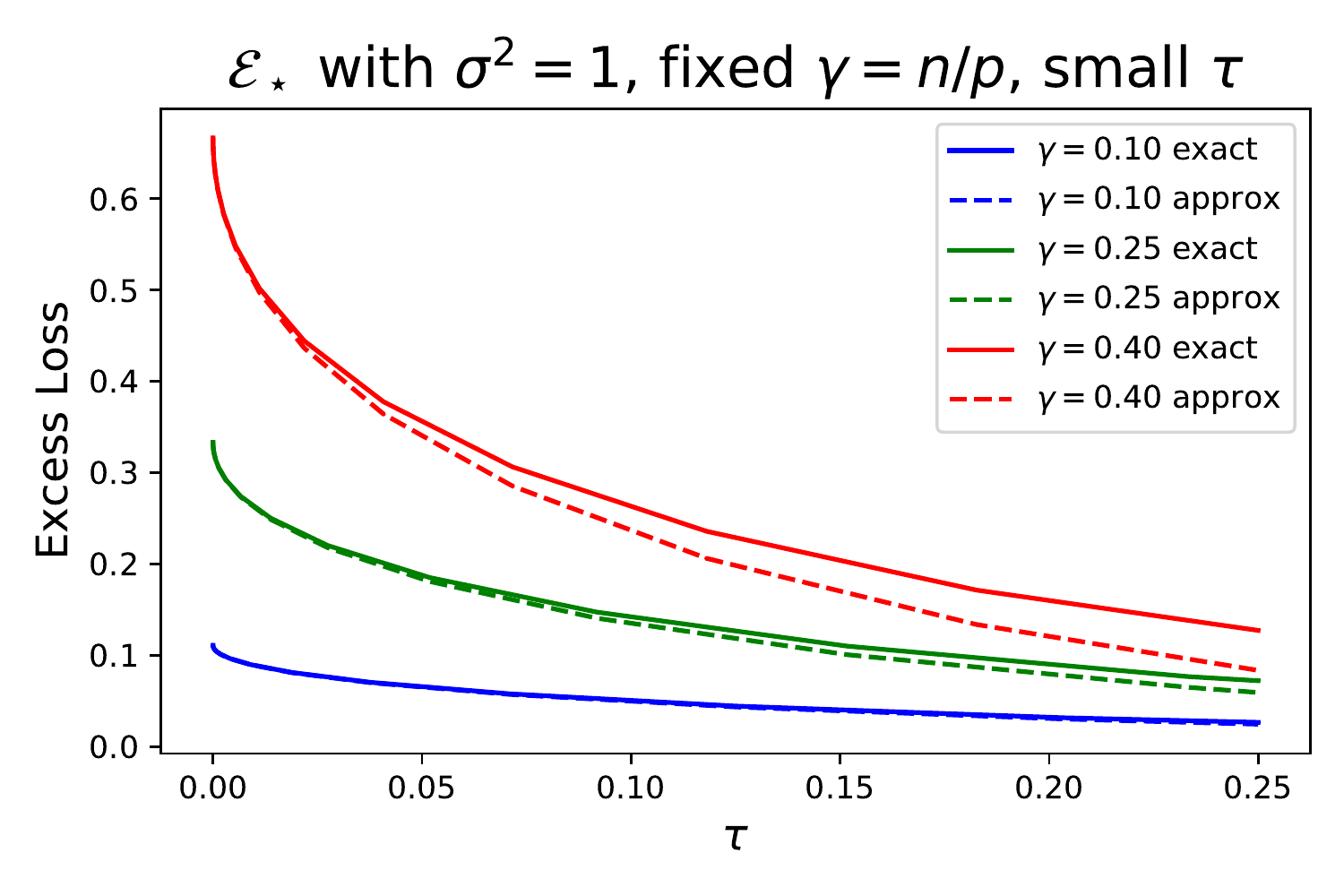}
    \caption{For fixed $\gamma = n/p$ we plot the exact value of $\minexcess(\ofitlevel)$ using Eq.\ (\ref{eq:min_excess_loss_analytic}) [solid] and the first-order Taylor approximation around $\ofitlevel = 0$ from Eq.\ \ref{eq:mp-taylor} [dashed]. Observe that as $\gamma \to 1$ the approximation becomes less tight for larger values of $\ofitlevel$.}
    \label{fig:mp-law-taylor}
\end{figure}
\begin{theorem}[MP lower bound for small $\tau$]\label{thm:mp-law-local}
For fixed $\gamma \in (0, 1)$ and $\ofitlevel \in [0, 1 - \gamma]$ we have
\begin{equation}\label{eq:mp-taylor}
    \minexcess(\ofitlevel, \gamma) \ge  \frac{\sigma^2 \gamma}{1 - \gamma}\qty(1 - \frac{\sqrt{\ofitlevel}}{\sqrt{1 - \gamma}})^2.
\end{equation}
Moreover, the lower bound is tight up to $o(\sqrt{\ofitlevel})$ error.
\end{theorem}
\begin{proof}
 In the proof we will often suppress the dependence on $\gamma$ and express dependence on $\ofitlevel$ in terms of $t := \sqrt{\ofitlevel}$. In particular, recalling the functions $E$, $f$, and $f^{-1}$ defined in Eqs.\ (\ref{eq:func_E}), (\ref{eq:func_f}), and (\ref{eq:inverse}) we use
\begin{alignat*}{3}
    \minexcess(t) &:= \minexcess(\ofitlevel, \gamma),~~~~ f^{-1}(t) &&:= f^{-1}(\ofitlevel, \gamma)\\
    E(\ridgeparam) &:= E(\ridgeparam, \gamma),~~~~ f(\ridgeparam) &&:= f(\ridgeparam, \gamma),
\end{alignat*}
and as a result 
Eq.\ (\ref{eq:inverse}) becomes
\begin{equation}\label{eq:inverse_reparam}
    f(f^{-1}(t)) = t^2.
\end{equation}
The first-order Taylor approximation of $\sqrt{\minexcess(t)}$ around $t=0$ is given by 
\begin{equation}\label{eq:first-order}
P_1(t) := \sqrt{\minexcess(0)} + t \cdot \dv{t}\sqrt{\minexcess(t)}\Bigr|_{t=0} = \sqrt{\minexcess(0)} +  \frac{t}{2 \sqrt{\minexcess(0)}}  \cdot \minexcess'(0).
\end{equation}
where as $t \to 0^+$
\begin{equation}\label{eq:remainder}
    \sqrt{\minexcess(t)} = P_1(t) + o(t).
\end{equation}
Note that since $\lim_{t \to 0^+} f^{-1}(t, \gamma) = 0$ for any $\gamma$, by Lemma \ref{lem:stieltjes_at_zero} we
\[
\minexcess(0) = \sigma^2 \gamma \lim_{\ridgeparam \to 0^+} [m(-\ridgeparam; \gamma) - \ridgeparam m'(-\ridgeparam; \gamma)] = \sigma^2 \gamma \lim_{\ridgeparam \to 0^+} m(-\ridgeparam; \gamma) = \frac{\sigma^2 \gamma}{1 - \gamma}.
\]
Furthermore letting $\ridgeparam = f^{-1}(t)$, we have by the chain rule
\begin{equation}\label{eq:chain_rule_error}
    \minexcess'(t) = \sigma^2 \gamma \dv{E(f^{-1}(t))}{t} = \sigma^2 \gamma \dv{E(\ridgeparam)}{\ridgeparam} \dv{f^{-1}(t)}{t}.
\end{equation}
We can directly calculate the first term
\begin{align*}
    \dv{E(\ridgeparam)}{\ridgeparam} = -2m(-\ridgeparam; \gamma) + \ridgeparam m''(-\ridgeparam; \gamma).
\end{align*}
By differentiating both sides of Eq.\ (\ref{eq:inverse_reparam}) with respect to $t$ we get
\[
\dv{f^{-1}(t)}{t}\dv{f(\ridgeparam)}{\ridgeparam} = 2t.
\]
Therefore rearranging the above gives
\begin{align*}
    \dv{f^{-1}(t)}{t} = \frac{2t}{f'(\ridgeparam)} &= \frac{2 t}{2 \ridgeparam m'(-\ridgeparam; \gamma) - \ridgeparam^2 m''(-\ridgeparam; \gamma)}\\
    &= \frac{2 \ridgeparam \sqrt{m'(-\ridgeparam; \gamma)}}{2 \ridgeparam m'(-\ridgeparam; \gamma) - \ridgeparam^2 m''(-\ridgeparam; \gamma)}\\
    &=  \frac{2 \sqrt{m'(-\ridgeparam; \gamma)}}{2 m'(-\ridgeparam; \gamma) - \ridgeparam m''(-\ridgeparam; \gamma)}.
\end{align*}
Hence from Eq.\ (\ref{eq:chain_rule_error}) we see
\begin{equation}\label{eq:error_derivative}
    \minexcess'(t) = -2\sigma^2 \gamma \sqrt{m'(-\ridgeparam, \gamma)}.
\end{equation}
Observe that since $m'(-\ridgeparam; \gamma) \geq 0$, the derivative $\minexcess'(t) \leq 0$. Furthermore, since $\ridgeparam$ is increasing in $t$ and $m'(-\ridgeparam; \gamma)$ is decreasing in $\ridgeparam$, it follows that $m'(-\ridgeparam; \gamma)$ is decreasing in $t$. Thus $\minexcess'(t)$ is increasing, i.e. $\minexcess''(t) \geq 0$.
Using Lemma \ref{lem:stieltjes_prime_at_zero}
\[
    \minexcess'(0) = \lim_{\ridgeparam \to 0^+} -2\sigma^2 \gamma \sqrt{m'(-\ridgeparam, \gamma)} = -2\gamma \sigma^2 \frac{1}{(1-\gamma)^{3/2}}.
\]
Plugging this into the Taylor approximation $P_1(t)$ in Eq.\ (\ref{eq:first-order}) we get that
\begin{align*}
    P_1(t) &= \sigma\sqrt{\frac{\gamma}{1 - \gamma}}  -\sigma t \sqrt{\frac{1-\gamma}{\gamma}} \frac{\gamma}{(1 - \gamma)^{3/2}}\\
    &= \sigma \sqrt{\frac{\gamma}{1 - \gamma}}\qty(1 - t\frac{1}{\sqrt{1 - \gamma}}).
\end{align*}
We now show that $\sqrt{\minexcess(t)} \geq P_1(t)$ by showing that $\sqrt{\minexcess(t)}$ is convex on $[0, 1]$ for which it suffices to show that the second derivative is non-negative. Computing the second derivative, we get
\[
    \dv[2]{t} \sqrt{\minexcess(t)} = \frac{\minexcess(t)}{2\sqrt{\minexcess(t)}} - \frac{\minexcess'(t)}{4\minexcess(t)^{3/2}}.
\]
We can see that this is always non-negative since $\minexcess(t) \geq 0$ and from Eq.\ (\ref{eq:error_derivative}) and the accompanying remarks $\minexcess'(t) \leq 0$ and $\minexcess''(t) \geq 0$. Returning to the original $\ofitlevel$ parameterization, we have that for all $\ofitlevel \in [0, 1]$
\[
    \sqrt{\minexcess(\ofitlevel)} \geq P_1(\sqrt{\ofitlevel})
    = \sigma \sqrt{\frac{\gamma}{1 - \gamma}}\qty(1 - \sqrt{\frac{\ofitlevel}{1 - \gamma}}).
\]
Note that for $\ofitlevel \leq 1 - \gamma$, both sides of the above inequality are non-negative. Hence we can square both sides
yielding
\[
    \minexcess(\ofitlevel, \gamma) \geq \sigma^2\frac{\gamma}{1 - \gamma}\qty(1 - \sqrt{\frac{\ofitlevel}{1 - \gamma}})^2
\]
which was the desired lower bound. The fact that the bound is tight up to $o(\sqrt{\ofitlevel})$ follows from Eq.\ (\ref{eq:remainder}).
\end{proof}

\begin{remark}\label{rmk:square-root-sharp}
Observe that from the above computations the first-order expansion of $\minexcess(t)$ is given by
\[
\minexcess(0) + \sqrt{\ofitlevel} \minexcess'(0) = \sigma^2 \sqrt{\frac{\gamma}{1 - \gamma}}\qty(1 - 2\sqrt{\frac{\ofitlevel}{1 - \gamma}}) < \sigma^2\frac{\gamma}{1 - \gamma}\qty(1 - \sqrt{\frac{\ofitlevel}{1 - \gamma}})^2 \leq \minexcess(\ofitlevel)
\]
which shows that expanding the square-root of the minimum excess error yields a tighter bound.
\end{remark}
Note that the bound in Eq.\ (\ref{eq:mp-taylor}) matches the lower bound in Theorem \ref{thm:mp-law-lower-bound} as $\gamma \to 0$. Interestingly, this bound has the form of a multiplicative factor of the minimum interpolation $(\ofitlevel = 0)$ loss and is only valid for small $\ofitlevel$. We will in fact show that no bound of that form can capture the behavior of the excess loss for arbitrarily small $\ofitlevel > 0$ as $\gamma \to 1$ in Theorem \ref{thm:peak}. In particular, we now show a discontinuity at $\ofitlevel = 0$ when considering the peak $n = p$.
\begin{theorem}[Excess Loss at Peak]\label{thm:peak}
For $\gamma = 1$ and $\ofitlevel \in (0, 1)$,
\[
\frac{\minexcess(\ofitlevel, 1)}{\sigma^2} = \frac{1}{4 \ofitlevel} + \frac{\ofitlevel}{4} - \frac{1}{2}
\]
hence as $\ofitlevel \to 0$, 
\[\minexcess(\ofitlevel, 1) = \Theta(\ofitlevel^{-1}).\]
\end{theorem}
\begin{proof}
Consider the Stieltjes transform $m(-z; \gamma)$ and its derivative. At the interpolation peak $\gamma = 1$, we have
\begin{align*}
    m(-z; 1) &= \frac{1}{2}\qty(\sqrt{1 + 4/z} - 1)\\
    m'(-z; 1) &= -\frac{z + 2}{2z \sqrt{z^2 + 4z}} + \frac{\sqrt{z^2 + 4z}}{2z^2}.
\end{align*}
One can then check that
\[
z^2 m'(-z; 1) = \frac{1}{\sqrt{1 + 4/z}}.
\]
Hence the fixed point equation Eq.\ (\ref{eq:asymp_fixed_point}) can be written as
\[
\ofitlevel = \ridgeparam^2 m'(-\ridgeparam; 1) = \frac{1}{\sqrt{1 + 4/\ridgeparam}}
\]
which implies that $\ridgeparam$ satisfies the following
\begin{align*}
    \sqrt{1 + 4 / \ridgeparam} &= \frac{1}{\ofitlevel},\\
    \frac{1}{\ridgeparam} &= \frac{1}{4}\qty(\frac{1}{\ofitlevel^2} - 1).
\end{align*}
Using the above relations, we can compute the minimum excess error using Eq.\ (\ref{eq:min_excess_loss_analytic})
\begin{align*}
    \frac{\minexcess(\ofitlevel, 1)}{\sigma^2} &= m(-\ridgeparam; 1) - \ridgeparam m'(-\ridgeparam; 1)\\
    &= \frac{1}{2}\qty(\sqrt{1 + 4/\ridgeparam} - 1) - \frac{\ofitlevel}{\ridgeparam}\\
    &= \frac{1}{2}\qty(\frac{1}{\ofitlevel} - 1)- \frac{1}{4}\qty(\frac{1}{\ofitlevel} - \tau) \\
    &= \frac{1}{4 \ofitlevel} + \frac{\ofitlevel}{4} - \frac{1}{2}.
\end{align*}
\end{proof}
\begin{remark}
As mentioned earlier, a consequence of the above is that we cannot have a lower bound of the form $\minexcess(\ofitlevel; n, p) \geq \minexcess(0; n, p) \cdot f(\ofitlevel)$ for some function $f$ which satisfies $f(\ofitlevel) > 0$ for $\ofitlevel > 0$ since for any $\ofitlevel > 0$,
\[
\lim_{\gamma \to 1} \frac{\minexcess(\ofitlevel, \gamma)}{\minexcess(0, \gamma)} = 0.
\]
\end{remark}


%% file: sections/conclusion.tex
\section{Conclusion}\label{sec:conclusions} 

In this work we demonstrated a trade-off between the expected loss, empirical loss, and the number of parameters for general linear models. 
In particular we have shown that near-optimal algorithms output models that are either classical (with empirical loss approaching the noise level) or have significant excess over-parameterization, i.e., have many more parameters than the number needed to fit the training data. This trade-off is universal as it is non-asymptotic,  holds for any algorithm, and any data distribution (under mild non-degeneracy assumptions).

We also provided a more precise asymptotic lower bound under Marchenko-Pastur distributional assumptions \revised{near the classical double descent peak where the amount of overparametrization is just enough to interpolate the data}. Remarkably \revised{however, as the level of overparametrization increases the minimum excess loss exactly matches the universal bound, demonstrating the tightness of the bound}.

The open questions that remain include \remove{resolving the discrepancy between the universal lower bound and the lower bound under Marchenko-Pastur asymptotics in terms of $\tau$, and more importantly,} extending our results to more general non-linear parametric families and to classification settings.

%% file: sections/acknowledgments.tex
\section*{Acknowledgements}
\revised{The authors would like thank Amirhesam Abedsoltan for finding an error in a previous version of the proof of Theorem \ref{thm:general_lower_bound}. Correcting the proof led to an improved lower bound which is now  tight. We also thank the anonymous reviewers for insightful comments.} We are grateful for support from the National Science Foundation (NSF) and the Simons Foundation for the Collaboration on the Theoretical Foundations of Deep Learning\footnote{\url{https://deepfoundations.ai/}} through awards DMS-2031883 and \#814639 as well
as NSF IIS-1815697 and the TILOS institute (NSF CCF-2112665). NG would also like to acknowledge support from the NSF RTG Grant \#1745640.

%% file: sections/appendix.tex
\section{Missing Proofs from Section \ref{sec:main_proof}}\label{app:excess_lin_loss}
In this section of the appendix, we supply missing proofs for some the auxiliary results used in the proof of Theorem \ref{thm:general_lower_bound} in Section \ref{sec:main_proof}. Recall that we denote the optimal linear predictor as $\beta_\star := \argmin_{\bbeta \in \R^p} R(\beta)$.
\begin{lemma}[Optimal Linear Predictor]\label{lem:opt_lin_app}
Define $\bbeta_\star$ as in Eq.\ (\ref{eq:opt_lin_predictor}). Then 
\begin{enumerate}
    \item $\beta_\star$ is the orthogonal projection of $f_\star$ onto the subspace of linear functions $\cH = \{\beta(\bx) : \bbeta \in \R^p\}$ in $L^2$,\label{claim:opt_lin_proj}
    
    \item $\excess(f) := R(f) - R(f_\star) = \E_{\bx}[(f(\bx) - f_\star(\bx))^2$,\label{claim:excess}
    
    \item $\excess^{\lin}(\beta) := R(\beta) - R(\beta_\star) = \E_{\bx}[(\beta(\bx) - \beta_\star(\bx))^2]$.\label{claim:excess_lin}
\end{enumerate}
\end{lemma}
\begin{proof}
By definition of the optimal linear predictor
\begin{align*}
    \beta_\star &= \argmin_{\beta \in \cH} R(\beta)\\
    &= \argmin_{\beta \in \cH} \E_{(\bx, y)}[(y - \beta(\bx))^2]\\
     &= \argmin_{\beta \in \cH} \E_{(\bx, y)}[(y - f_\star(\bx))^2 + (f_\star(\bx) - \beta(\bx))^2] && (\text{since } \E_{y \mid \bx}[y - f_\star(\bx)] = 0)\\
     &= \argmin_{\beta \in \cH} \E_{\bx}[(f_\star(\bx) - \beta(\bx))^2],
\end{align*}
which shows Claim \ref{claim:opt_lin_proj}. For Claim \ref{claim:excess}
\begin{align*}
    R(f) - R(f_\star) &= \E_{(\bx, y)}[(y - f(\bx))^2] - \E_{(\bx, y)}[(y - f_\star(\bx))^2]\\
    &= \E_{(\bx, y)}[(f(\bx) - f_\star(\bx))^2] - 2\E_{(\bx, y)}[(f(\bx) - f_\star(\bx))(y - f_\star(\bx))]\\
    &= \E_{\bx}[(f(\bx) - f_\star(\bx))^2].
\end{align*}
For Claim \ref{claim:excess_lin}
\begin{align*}
    R(\beta) - R(\beta_\star) &= \excess(\beta) - \excess(\beta_\star)\\
    &=\E_{\bx} [(f_\star(\bx) - \beta(\bx))^2 - (f_\star(\bx) - \beta_{\star}(\bx))^2]\\
    &= \E_{\bx} [(\beta(\bx) - \beta_\star(\bx))^2] - 2\E_{\bx}[(f_\star(\bx) - \beta_{\star}(\bx))(\beta(\bx) - \beta_\star(\bx))]\\
    &= \E_{\bx} [(\beta(\bx) - \beta_\star(\bx))^2]
\end{align*}
where the second equality holds from Claim \ref{claim:excess} and the last equality holds from Claim \ref{claim:opt_lin_proj} in Lemma \ref{lem:opt_lin_app}.
\end{proof}
\begin{lemma}[Excess Linear Loss]
The excess loss satisfies the following lower bound
\[
\excess(\beta) \geq \excess^{\lin}(\beta) = \norm{\bSigma^{1/2}(\bbeta - \bbeta_\star)}^2.
\]
\end{lemma}
\begin{proof}
It is clear that $\excess(\beta) \geq \excess^{\lin}(\beta)$. Using Lemma \ref{lem:opt_lin_app} and the definition of $\bSigma$
\begin{align*}
\excess(\beta) &= \E_{\bx}(\beta(\bx) - \beta_{\star}(\bx))^2 \\
&=\E_{\bx}(\bbeta^\sT \phi(\bx) - \bbeta_{\star}^\sT \phi(\bx))^2\\
&= (\bbeta - \bbeta_\star)^\sT \E_{\bx}[\phi(\bx)\phi(\bx)^\sT](\bbeta - \bbeta_\star)\\
&= \norm{\bSigma^{1/2}(\bbeta - \bbeta_\star)}^2.
\end{align*}
\end{proof}
\noindent Recall that we defined the random vectors
\[
\noise = (y_i - f_\star(\bx_i))_{i \in [n]} \in \R^n,~~~~ \linnoise = (f_\star(\bx_i) - \beta_\star(\bx_i))_{i \in [n]} \in \R^n
\]
and let $\effnoise = \noise + \linnoise$.

\begin{lemma}[Expectation Over Noise]
Let $f : \R^{n \times d} \to \S^n_+$ be any PSD matrix valued function. Then,
\begin{equation}\label{eq:expected_noise_lb}
    \E_{\ytrain, \Xtrain} [\effnoise^\sT f(\Xtrain) \effnoise] \geq \sigma^2 \E_{\Xtrain} \Tr(f(\Xtrain)).
\end{equation}
Moreover, if $\Var(y \midsmall \bx) = \sigma^2$ almost surely and $f_\star = \beta_\star$ then the above is an equality.
\end{lemma}
\begin{proof}
Recalling that $\effnoise = \noise + \linnoise$. By definition of $\noise$ we have that
\[
\E_{\ytrain \mid \Xtrain}[\noise] = \bzero
\]
and by the assumption on the noise variance we have almost surely
\begin{equation}\label{eq:noise_covariance}
    \E_{\ytrain \mid \Xtrain}[\noise \noise^\sT] = \diag((\Var[y \mid \bx_i])_{i \in [n]}) \succeq \sigma^2 \id.
\end{equation}
Using these observations we have 
\begin{align*}
    \E_{\ytrain \mid \Xtrain} [\effnoise^\sT f(\Xtrain) \effnoise] &= \E_{\ytrain \mid \Xtrain} [\noise^\sT f(\Xtrain) \noise] + 2\E_{\ytrain \mid \Xtrain} [\linnoise^\sT f(\Xtrain) \noise] + \E_{\ytrain \mid \Xtrain} [\linnoise^\sT f(\Xtrain) \linnoise]\\
    &= \Tr(f(\Xtrain) \E_{\ytrain \mid \Xtrain} [\noise \noise^\sT]) + 2 \linnoise^\sT f(\Xtrain) \E_{\ytrain \mid \Xtrain}[\noise] + \linnoise^\sT f(\Xtrain) \linnoise\\
    &= \Tr(f(\Xtrain) \E_{\ytrain \mid \Xtrain} [\noise \noise^\sT]) + \linnoise^\sT f(\Xtrain) \linnoise \\
    &\geq \sigma^2 \Tr(f(\Xtrain)) + \linnoise^\sT f(\Xtrain) \linnoise \\
    &\geq \sigma^2 \Tr(f(\Xtrain)),
\end{align*}
where we used Eq.  (\ref{eq:noise_covariance}) in the first inequality and we used the assumption that $f(\bX)$ is a PSD matrix for the last inequality.
By iterating expectations we have
\[
\E_{\ytrain,\Xtrain} [\effnoise^\sT f(\Xtrain) \effnoise] = \E_{\Xtrain} \E_{\ytrain \mid \Xtrain} [\effnoise^\sT f(\Xtrain) \effnoise] \geq \sigma^2 \E_{\Xtrain} \Tr(f(\Xtrain))
\]
as desired. Note that if $\Var(y \midsmall \bx) = \sigma^2$ then Eq.  (\ref{eq:noise_covariance}) becomes an equality and if $f_\star = \beta_\star$ then $\linnoise = \bzero$ and so it easy to see that Eq.  (\ref{eq:expected_noise_lb}) is an equality as well.
\end{proof}
\newpage
\section{Marchenko-Pastur Law}\label{sec:app-mp}
In this section, we will give some definitions, facts, and results concerning the Marchenko-Pastur Law.
\begin{definition}[Marchenko-Pastur Law]
    The Marchenko-Pastur Law with aspect ratio $\gamma$, denoted $\MP(\gamma)$, is given the probability distribution
    \[
        \qty(1 - \frac{1}{\gamma})_+ \delta_0 + \frac{1}{\gamma 2 \pi x} \sqrt{(b - x)(x - a)} 
    \]
    where
    \[
    a = (1 - \sqrt{\gamma})\revised{^2},~~~~ b= (1 + \sqrt{\gamma})\revised{^2}.
    \]
\end{definition}

\begin{definition}[Stieltjes Transform]
For $z > 0$ and $\gamma > 0$, the Stieltjes transform of the Marchenko-Pastur law is
\[
m(-z; \gamma) := \int \frac{1}{s + z} \dd{\MP_\gamma(s)} = \frac{-(1 - \gamma + z) + \sqrt{(1 - \gamma + z)^2 + 4 \gamma z}}{2 \gamma z}
\]
and its derivative satisfies
\begin{align*}
    m'(-z; \gamma) &= \int \frac{1}{(s + z)^2} \dd{\MP_\gamma(s)}\\
     &= \frac{-(1-\gamma+z) - 2\gamma}{2 \gamma z \sqrt{(1 - \gamma + z)^2 + 4 \gamma z}} + \frac{1}{2 \gamma z} + \frac{\sqrt{(1 - \gamma + z)^2 + 4 \gamma z}-(1 - \gamma + z)}{2z^2 \gamma}.
\end{align*}
\end{definition}

\begin{lemma}\label{lem:stieltjes_at_zero}
For any $\gamma > 0$, 
\[
\lim_{z \to 0^+} m(-z; \gamma) = \frac{1}{1 - \gamma}
\]
\end{lemma}
\begin{proof}
This just follows from a direct computation
\begin{align*}
    \lim_{z \to 0^+} m(-z; \gamma) &= \lim_{z \to 0^+} \frac{-(1 - \gamma + z) + \sqrt{(1 - \gamma + z)^2 + 4 \gamma z}}{2 \gamma z} \\
    &= \lim_{z \to 0^+} \frac{-(1 - \gamma + z)^2 + (1 - \gamma + z)^2 + 4 \gamma z}{2 \gamma z [(1 - \gamma + z) + \sqrt{(1 - \gamma + z)^2 + 4 \gamma z}]} \\
    &= \lim_{z \to 0^+} \frac{2}{(1 - \gamma + z) + \sqrt{(1 - \gamma + z)^2 + 4 \gamma z}} = \frac{1}{1 - \gamma}.
\end{align*}
\end{proof}

\begin{lemma}\label{lem:stieltjes_prime_at_zero}
For any $\gamma > 0$, 
\[
\lim_{z \to 0^+} m'(-z; \gamma) = \frac{1}{(1 - \gamma)^3}
\]
\end{lemma}
\begin{proof}
By the definition of derivative
\begin{align*}
    \lim_{z \to 0^+} m'(-z; \gamma) &= \lim_{z \to 0^+} \frac{m(-z; \gamma) - m(0; \gamma)}{-z} \\
    &= \lim_{z \to 0^+} \frac{1}{z}\qty(\frac{(1 - \gamma + z) - \sqrt{(1 - \gamma + z)^2 + 4 \gamma z}}{2 \gamma z} + \frac{1}{1 - \gamma}) \\ 
    &= \lim_{z \to 0^+} \frac{(1 - \gamma)(1 - \gamma + z) - (1 - \gamma)\sqrt{(1 - \gamma + z)^2 + 4 \gamma z} + 2 \gamma z}{2 \gamma(1-\gamma)z^2} \\
    &= \lim_{z \to 0^+} \frac{(1 - \gamma)^2 + z(1 + \gamma) - (1 - \gamma)\sqrt{(1 - \gamma + z)^2 + 4 \gamma z}}{2 \gamma(1-\gamma)z^2} \\
    &= \lim_{z \to 0^+} \frac{[(1 - \gamma)^2 + z(1 + \gamma)]^2 - (1 - \gamma)^2[(1 - \gamma + z)^2 + 4 \gamma z]}{2 \gamma(1-\gamma)z^2[(1 - \gamma)^2 + z(1 + \gamma) +(1 - \gamma)\sqrt{(1 - \gamma + z)^2 + 4 \gamma z} ]}
\end{align*}
Observe that the \revised{numerator} simplifies as
\begin{align*}
    &[(1 - \gamma)^2 + z(1 + \gamma)]^2 - (1 - \gamma)^2[(1 - \gamma + z)^2 + 4 \gamma z]\\ &= (1 - \gamma)^4 + 2z(1-\gamma)^2(1 + \gamma) + z^2(1 + \gamma)^2-[(1 - \gamma)^4 + 2z(1 - \gamma)^2(1+\gamma) + (1 - \gamma)^2 z^2]\\
    &= z^2[(1 + \gamma)^2 - (1 - \gamma)^2] = 4 z^2 \gamma.
\end{align*}
Hence
\begin{align*}
    \lim_{z \to 0^+} m'(-z; \gamma) &= \lim_{z \to 0^+} \frac{4 \gamma}{2 \gamma(1-\gamma)[(1 - \gamma)^2 + z(1 + \gamma) +(1 - \gamma)\sqrt{(1 - \gamma + z)^2 + 4 \gamma z} ]}\\
     &= \frac{2}{(1-\gamma)[(1 - \gamma)^2 +(1 - \gamma)\sqrt{(1 - \gamma )^2 } ]} = \frac{1}{(1 - \gamma)^3}.
\end{align*}
\end{proof}

\begin{lemma}\label{lem:stieltjes_gamma_increasing}
For any $z > 0$,
\[
\pdv{m(-z; \gamma)}{\gamma} \, \revised{ > } \, 0.
\]
\end{lemma}
\begin{proof}
We can rewrite the Stieltjes transform as follows
\begin{align*}
    m(-z; \gamma) &= \frac{-(1 - \gamma + z) + \sqrt{(1 - \gamma + z)^2 + 4 \gamma z}}{2 \gamma z}\\
    &= \frac{-(1 - \gamma + z) + \sqrt{(1 - \gamma + z)^2 + 4 \gamma z}}{2 \gamma z} \frac{(1 - \gamma + z) + \sqrt{(1 - \gamma + z)^2 + 4 \gamma z}}{(1 - \gamma + z) + \sqrt{(1 - \gamma + z)^2 + 4 \gamma z}}\\
    &= \frac{2}{(1 - \gamma + z) + \sqrt{(1 - \gamma + z)^2 + 4 \gamma z}}.
\end{align*}
Therefore it suffices to show that
\[
f(\gamma) = (1 - \gamma + z) + \sqrt{(1 - \gamma + z)^2 + 4 \gamma z}
\]
is decreasing in $\gamma$. Taking the derivative we see that
\[
f'(\gamma) = -1 + \frac{-(1 - \gamma + z) + 2z}{ \sqrt{(1 - \gamma + z)^2 + 4 \gamma z}}.
\]
Note that $f'(\gamma) \, \revised{ < } \, 0$ if and only if
\[
2z - (1 - \gamma + z) \, \revised{ < } \, \sqrt{(1 - \gamma + z)^2 + 4 \gamma z}.
\]
Squaring both sides and clearing terms this is equivalent to
\[
4z^2 - 4z(1 - \gamma + z)  \, \revised{ < } \, 4 \gamma z.
\]
Using the fact that $z > 0$ we can divide both sides by $4z$ to get
\[
z - (1 - \gamma + z) \, \revised{ < } \, \gamma
\]
which is clearly true.
\end{proof}
\newpage
\section{Auxiliary Lemmas}
\begin{lemma}[AM-HM inequality]\label{lem:AM-HM}
Let $x_1, \ldots, x_n \in \R$. Then
\[
\sum\limits_{i=1}^n \frac{1}{x_i} \geq n^2\qty(\sum\limits_{i=1}^n x_i)^{-1}.
\]
\end{lemma}
\begin{proof}
By the Cauchy-Schwarz inequality
\[
n^2 = \qty(\sum\limits_{i=1}^n \sqrt{x_i} \cdot \frac{1}{\sqrt{x_i}})^2 \leq \sum\limits_{i=1}^n \sqrt{x_i}^2 \cdot \sum\limits_{i=1}^n \frac{1}{\sqrt{x_i}^2},
\]
which after re-arranging gives the desired inequality.
\end{proof}

\revised{
\begin{lemma}[Chebyshev's Sum Inequality \cite{hardy1952inequalities}]\label{lem:chebyshev-sum-inequality}
    If $a_1 \leq a_2 \leq \ldots \leq a_n$ and $b_1 \geq \ldots \geq b_n$ then
    \[
        \frac{1}{n} \sum\limits_{i=1}^n a_i b_i \leq \qty(\frac{1}{n} \sum\limits_{i=1}^n a_i)\qty(\frac{1}{n} \sum\limits_{i=1}^n b_i).
    \]
\end{lemma}
}

\begin{lemma}[Concentration of Quadratic Forms, Lemma 7.6 from \cite{dobriban2018high}]\label{prop:convergence_quad_form}
For each positive integer $n$, let $\bA_n$ be a random $n \times n$ PSD matrix. Let $\xi_1, \xi_2, \ldots$ be a sequence of i.i.d random variables and denote $\bxi_n = (\xi_1, \ldots, \xi_n)$. Then if
\begin{itemize}
    \item $\sup_n \norm{\bA_n} = O(1)$
    \item $\E[\xi_1] = 0$, $\E[\xi_1^2] = \sigma^2$, and $\E[\xi_1^{4 + \eta}]$ for some $\eta > 0$
\end{itemize}
we have the following convergence
\[
\frac{1}{n} \bxi_n^\sT \bA_n \bxi_n - \frac{\sigma^2}{n} \Tr(\bA_n) \to 0
\]
almost surely.
\end{lemma}